\documentclass[12pt]{article}

\usepackage{natbib}
\usepackage{microtype}
\usepackage{graphicx}
\usepackage{subfigure}
\usepackage{booktabs} 
\usepackage{fullpage}

\usepackage[breaklinks,colorlinks,
            linkcolor=red,
            anchorcolor=blue,
            citecolor=blue
            ]{hyperref}

\usepackage{amsmath}
\usepackage{amssymb}
\usepackage{mathtools}
\usepackage{amsthm}

\usepackage[capitalize,noabbrev]{cleveref}

\theoremstyle{plain}
\newtheorem{theorem}{Theorem}[section]
\newtheorem{proposition}[theorem]{Proposition}
\newtheorem{lemma}[theorem]{Lemma}
\newtheorem{corollary}[theorem]{Corollary}
\theoremstyle{definition}
\newtheorem{definition}[theorem]{Definition}
\newtheorem{assumption}[theorem]{Assumption}
\theoremstyle{remark}
\newtheorem{remark}[theorem]{Remark}

\usepackage[textsize=tiny]{todonotes}

\usepackage{relsize}
\usepackage[utf8]{inputenc} 
\usepackage[T1]{fontenc}    
\usepackage{url}            
\usepackage{booktabs}       
\usepackage{amsfonts}       
\usepackage{nicefrac}       
\usepackage{microtype}      

\usepackage{smile}

\usepackage{bbm}
\usepackage{xcolor}
\usepackage{tikz}
\usetikzlibrary{calc}
\tikzset{stretch/.initial=1}
\usetikzlibrary{patterns}
\usetikzlibrary{arrows}

\newcommand{\brt}{\text{br}}
\newcommand{\normlambda}[1]{\left\lVert #1 \right\rVert_{\Lambda_h^{-1}}}
\newcommand{\abso}[1]{\left\lvert #1\right\rvert}
\def \T {\mathsf{T}}
\def\##1\#{\begin{align}#1\end{align}}
\def\$#1\${\begin{align*}#1\end{align*}}
\def \bellmanhat{\widehat{\mathbb{B}}_h }
\def \bellman {\mathbb{B}_h}
\def \expect {\mathbb{E}}
\def \pihat {\widehat{\pi}}

\def\valuefunhat{\widehat{V}}
\newcommand{\qfunhat}{\widehat{Q}}
\newcommand{\statep}{(s,a,b)}
\newcommand{\stateph}{(s_h,a_h,b_h)}
\newcommand{\statepht}{(s_h^{\tau},a_h^{\tau},b_h^{\tau})}
\def\expectgamei#1{\mathbb{E}_{\mathcal{D}\sim\mathcal{M_{\text{1}}}}\left[#1\right]}
\def\expectgameii#1{\mathbb{E}_{\mathcal{D}\sim\mathcal{M_{\text{2}}}}\left[#1\right]}
\newcommand{\normof}[1]{\left\lVert #1 \right\rVert }
\newcommand{\normop}[1]{\left\lVert #1 \right\rVert_{\text{op}} }
\newcommand{\bpc}{\overline{\pi},\overline{\nu}}

\def\algname{Pessimistic Minimax Value Iteration }

\def \d {\text{d}}
\def \mcs{\mathcal{S}}
\def \mcm{\mathcal{M}}
\def \mbp{\mathbb{P}}
\def \mcd{\mathcal{D}}
\def \nuhat{\widehat{\nu}}
\def \ri{\mathrm{RU}}
\newcommand\drawloop[4][]%
   {\draw[draw=red!75!black, shorten <=0pt, shorten >=0pt,#1]
      ($(#2)!\pgfkeysvalueof{/tikz/stretch}!(#2.#3)$)
      let \p1=($(#2.center)!\pgfkeysvalueof{/tikz/stretch}!(#2.north)-(#2)$),
          \n1= {veclen(\x1,\y1)*sin(0.5*(#4-#3))/sin(0.5*(180-#4+#3))}
      in arc [start angle={#3-90}, end angle={#4+90}, radius=\n1]%
   }

\title{\LARGE Pessimistic Minimax Value Iteration: Provably Efficient Equilibrium Learning from Offline Datasets}
\author{Han Zhong\thanks{The first three authors contributed equally.}~\thanks{Peking University. Email: \texttt{hanzhong@stu.pku.edu.cn}} \qquad Wei Xiong\thanks{The Hong Kong University of Science and Technology. Email: \texttt{wxiongae@connect.ust.hk}}  \qquad Jiyuan Tan\thanks{Fudan University. Email: \texttt{jiyuantan19@gmail.com}}  \qquad Liwei Wang\thanks{Peking University. Email: \texttt{wanglw@cis.pku.edu.cn}} \\ Tong Zhang\thanks{The Hong Kong University of Science and Technology. Email: \texttt{tongzhang@tongzhang-ml.org}} \qquad Zhaoran Wang\thanks{Northwestern University. Email: \texttt{zhaoranwang@gmail.com}} \qquad  Zhuoran Yang\thanks{Yale University. Email: \texttt{zhuoran.yang@yale.edu}}}
\date{}
\begin{document}
\maketitle




\begin{abstract}
  We study episodic two-player zero-sum Markov games (MGs) in the offline setting, where the goal is to find an approximate Nash equilibrium (NE) policy pair based on a dataset collected a priori. When the dataset does not have uniform coverage over all policy pairs, finding an approximate NE involves challenges in three aspects: (i) distributional shift between the behavior policy and the optimal policy, (ii) function approximation to handle large state space, and (iii)  minimax optimization for equilibrium solving. 
  We propose a pessimism-based algorithm, dubbed as pessimistic minimax value iteration (PMVI), which overcomes the distributional shift by constructing pessimistic estimates of the value functions for both players and outputs a policy pair by solving NEs based on the two value functions. Furthermore, we establish a data-dependent upper bound on the suboptimality which recovers a sublinear rate without the assumption on uniform coverage of the dataset. 
We also prove an information-theoretical lower bound, which  suggests that the data-dependent term in the upper bound is intrinsic.
  Our theoretical results also highlight a notion of ``relative uncertainty'', which characterizes the necessary and sufficient condition for achieving sample efficiency in offline MGs.  
  To the best of our knowledge, we provide the first nearly minimax optimal result for offline MGs with function approximation. 
\end{abstract}

\section{Introduction}
Reinforcement learning (RL) has recently achieved tremendous empirical success, including Go \citep{silver2016mastering, silver2017mastering}, Poker \citep{brown2019superhuman}, robotic control \citep{kober2013reinforcement}, and Dota \citep{berner2019dota}, many of which involve multiple agents. 
 RL system with multiple agents acting in a common environment is referred to as multi-agent RL (MARL) where each agent aims to maximize its 
 own long-term return by interacting with the environment and other agents \citep{zhang2021multi}. 
 Two key components of these successes are \textit{function approximation} and \textit{efficient simulators}. For modern RL applications with large state spaces, function approximations such as neural networks are used to approximate the value functions or the policies and contributes to the generalization across different state-action pairs. Meanwhile, an efficient simulator serves as the environment which allows the agent to collect millions to billions of trajectories for the training process. 

However, for various scenarios, e.g., healthcare \citep{pan2017agile} and auto-driving \citep{wang2018supervised} where either collecting data is costly and risky, or online exploration is not possible \citep{fu2020d4rl}, it is far more challenging to apply (MA)RL methods in a trial-and-error fashion. To tackle these issues, offline RL aims to learn a good policy from a pre-collected dataset without further interacting with the environment. Recently, there has been impressive progress in the theoretical understanding about single-agent offline RL \citep{jin2020pessimism,rashidinejad2021bridging,zanette2021provable,xie2021bellman,yin2021towards, uehara2021pessimistic}, indicating that pessimism is critical for designing provably efficient offline algorithms. More importantly, these works demonstrate that the necessary and sufficient condition for achieving sample efficiency in offline MDP is the single policy (optimal policy) coverage. 
That is, it suffices for 
the offline dataset to have  good coverage over the trajectories induced by the optimal policy. 

In offline MARL for zero-sum Markov games, agents are not only facing the challenges of unknown environments, function approximation, and the distributional shift between the behavior policy and the optimal policy, but also challenged by the sophisticated minimax optimization for equilibrium solving. Due to these challenges, theoretical understandings of offline MARL remains elusive. In particular, the following questions remain open: 


\begin{center}
    (i) Can we design sample-efficient equilibrium learning algorithms in offline MARL? \\ (ii) What is the necessary and sufficient condition for achieving sample efficiency in offline MARL? 
\end{center}
To this end, focusing on the two-player zero-sum and  finite-horizon Markov Game (MG) with linear function approximation, we provide positive answers to the above two questions. 
Our contribution is threefold: 


\begin{itemize}
    \item For the two-player zero-sum MG with linear function approximation, we propose a computationally efficient algorithm, dubbed as pessimistic minimax value iteration (PMVI), which features the pessimism mechanism. 
    \item We introduce a new notion of ``relative uncertainty'', which depends on the offline dataset and $(\pi^*, \nu) \cup (\pi, \nu^*)$, where $(\pi^*, \nu^*)$ is an NE and $(\pi, \nu)$ are arbitrary policies. Furthermore, we prove that the suboptimality of PMVI can be bounded by \emph{relative uncertainty} up to multiplicative factors involving the dimension and horizon, which further implies that ``low relative uncertainty'' is the \emph{sufficient} condition for NE finding in the offline linear MGs setting. Meanwhile, by constructing a counterexample, we prove that, unlike the single-agent MDP where the single policy (optimal policy) coverage is enough, it is impossible to learn an approximate NE by the dataset only with the single policy pair (NE) coverage property.
    \item We also investigate the necessary condition for NE finding in the offline linear MGs setting. We demonstrate that the \emph{low relative uncertainty} is exactly the \emph{necessary} condition by showing that the relative uncertainty is the information-theoretic lower bound. This lower bound also indicates that PMVI achieves minimax optimality up to multiplicative factors involving the dimension and horizon. 
\end{itemize}
In summary, we propose the first \emph{computationally efficient} algorithm for offline linear MGs which is \emph{minimax optimal} up to multiplicative factors involving the dimension and horizon. More importantly, we figure out that \emph{low relative uncertainty} is the \emph{necessary and sufficient} condition for achieving sample efficiency in offline linear MGs setup. 

\subsection{Related Work}

There is a rich literature on MG \citep{shapley1953stochastic} and RL. Due to space constraint, we focus on reviewing the theoretical works on two-player zero-sum MG and offline RL.

\vspace{4pt}
\noindent{\textbf{Two-player zero-sum Markov game.}}  
There has been an impressive progress for online two-player zero-sum MGs, including the tabular MG \citep{bai2020provable, xie2020learning, bai2020near, liu2021sharp}, and MGs with linear function approximation \citep{xie2020learning, chen2021almost}. Beyond these two settings, \citet{jin2021power} and \citet{huang2021towards} consider the two-player zero-sum MG with general function approximation and the proposed algorithms can further solve MGs with kernel function approximation, MGs with rich observations, and kernel feature selection. For offline sampling oracle,   \citet{abe2020off} considers  offline policy evaluation under the strong uniform concentration assumption. 

\vspace{4pt}
\noindent{\textbf{Offline RL.}} The study of the offline RL (also known as batch RL), has a long history. In the single-agent setting, the prior works typically require a strong dataset coverage assumption \citep{precup2000eligibility, antos2008learning, levine2020offline}, which is impractical in general, particularly for the modern RL problems with large state spaces. Recently,    \citet{jin2020pessimism} takes  a   step towards identifying the minimal dataset assumption that empower provably efficient offline learning. 
In particular, it shows that pessimism principle allows efficient offline learning under a much weaker assumption which only requires a sufficient coverage over the optimal policy. After \citet{jin2020pessimism}, a line of work \citep{rashidinejad2021bridging, yin2021towards, uehara2021representation, zanette2021provable, xie2021bellman, uehara2021pessimistic} leverages the principle of pessimism to design offline RL algorithms, both in the tabular case and in the case with general function approximation.  These methods are not only more robust to the violation of dataset coverage assumption, but also provide non-trivial theoretical understandings of the offline learning, which are of independent interests. Despite the rich literature on single-agent offline RL, the extension to the MARL is still challenging. 

To the best of our knowledge, the current work on sample-efficient equilibrium finding in offline MARL is only \citet{zhong2021can} and \citet{cui2022offline}. In particular,  \citet{zhong2021can} studies the general-sum MGs with leader-follower structure and aims to find the Stackelberg-Nash equilibrium, but we focus on finding the NE in two-player zero-sum MGs with symmetric players. 
Our work is most closely related to the concurrent work \citet{cui2022offline}, which we discuss in detail below.


\vspace{4pt}
\noindent{\textbf{Comparison with \citet{cui2022offline}.}}
Up to now, the concurrent work \citet{cui2022offline} seems to provide  the only analysis on tabular two-player zero-sum MG in the offline setting.
We comment the similarities and differences between two works as follows.

In terms of algorithms, both PMVI (Algorithm~\ref{alg1}) in this paper and algorithms proposed in \citet{cui2022offline} are pessimism-type algorithms and computationally efficient. Since tabular MG is a special case of linear MG, our algorithm can naturally be applied to the tabular setting and achieve sample efficiency under the same coverage assumption. 

In terms of theoretical results, our work can be compared to \citet{cui2022offline} in the following aspects.     First,  both this work and \citet{cui2022offline} figure out the necessary and sufficient condition for achieving sample efficiency in (linear) MGs. Specifically, we introduce a new notion of \emph{relative uncertainty} and prove that the \emph{low relative uncertainty} is the necessary and sufficient condition for achieving sample efficiency in (linear) MGs.  \citet{cui2022offline} proposes a similar notion called \emph{unilateral concentration} and obtains similar results.    Second,  by constructing slightly different hard instances, both this work and \citet{cui2022offline} show that the single policy (NE) coverage assumption is not enough for NE identification in MGs.
  Third,  this work and \citet{cui2022offline} achieve near-optimal results in the linear setting and tabular setting, respectively. 
Finally,  the information-theoretic lower bound in \citet{cui2022offline} can be implied by that  for single-agent MDP. 
In contrast, our information-theoretic lower bound is construction-based and is a non-trivial extension from single-agent MDP. 


\section{Preliminaries}
In this section, we formally formulate our problem, and introduce preliminary concepts used in our paper. 

\subsection{Two-Player Zero-Sum Markov Game}
We consider a two-player zero-sum, finite-horizon MG where one agent (referred to as the max-player) aims to maximize the total reward while the other agent (referred to as the min-player) aims to minimize it. The game is defined as a tuple $\mathcal{M}\left(H, \mathcal{S}, \mathcal{A}_1, \mathcal{A}_2, r, \mathbb{P} \right)$ where $H$ is the number of steps in each episode, $\mathcal{S}$ is the state space, $\mathcal{A}_1, \mathcal{A}_2$ are the action spaces of the two players, respectively, $\mathbb{P} = \{\mathbb{P}_h\}_{h=1}^H$ is the transition kernel where $\PP_h(\cdot|s,a,b)$ is the distribution of the next state given the state-action pair $(s,a,b)$ at step $h$, $r = \{r_h(s,a,b)\}_{h=1}^H$ is the reward function\footnote{For ease of presentation, we consider deterministic reward. Our results immediately generalize to the stochastic reward function case.}, where $r_h(s,a,b) \in [0,1]$ is the reward given the state-action pair $(s,a,b)$ at step $h$. We assume that for each episode, the game starts with a fixed initial state $x \in \cS$ and it can be straightforwardly generalized to the case where the initial state is sampled from some fixed but unknown distribution.

\vspace{4pt}
\noindent{\textbf{Policy and Value functions.}} Let $\Delta(\cX)$ be the probability simplex over the set $\cX$. A Markov policy of the max-player is a sequence of functions $\pi = \{\pi_h:\cS \to \Delta(\cA_1)\}$ where $\pi_h(s)$ is the distribution of actions taken by the max-player given the current state $s$ at step $h$. Similarly, we can define the Markov policy of the min-player by $\nu = \{\nu_h: \cS \rightarrow \Delta(\cA_2)\}$. Given a policy pair $(\pi, \nu)$, the value function $V_h^{\pi, \nu}: \cS \to \RR$ and the Q-value function $Q_h^{\pi,\nu}: \cS \times \cA_1 \times \cA_2 \to \RR$ at step $h$ are defined by
$$
\begin{aligned}
V_{h}^{\pi, \nu}(s_h):=&\mathbb{E}_{\pi, \nu}\left[\sum_{h'=h}^{H} r_{h'}(s_{h'},a_{h'},b_{h'}) \middle| s_h\right],\\
Q_{h}^{\pi, \nu}(s_h, a_h, b_h):=&\mathbb{E}_{\pi, \nu}\left[\sum_{h'=h}^{H} r_{h'}(s_{h'}, a_{h'}, b_{h'}) \middle| s_{h}, a_h, b_h\right],
\end{aligned}
$$
where the expectation is taken over the randomness of the environment and the policy pair. We define the Bellman operator $\bellman$ for any function $V: \cS \to \RR$ as 
\begin{align}\label{bellman operator}
        \bellman V(s,a,b)  =  \expect\left[r_h(s, a, b) +V(s_{h+1})\middle|(s_h,a_h,b_h)=(s,a,b)\right].  
\end{align}
It is not difficult to verify that the value function and Q-value function satisfy the following Bellman equation:
\begin{equation}
    Q_h^{\pi,\nu}(s,a,b) = (\bellman V_{h+1}^{\pi,\nu})(s,a, b). \label{eq:bellmam_eq} 
\end{equation}

        \subsection{Linear Markov Game}
        We consider a family of MGs whose reward functions and transition kernels possess a linear structure.
        \begin{assumption}[Linear MGs  \citep{xie2020learning}] \label{assumption:linear:MG}
        For each $(s,a,b) \in \cS \times \cA_1 \times \cA_2$, and $h \in [H]$, we have
        \begin{equation}
        \begin{aligned}
        r_{h}(x, a, b)=\phi(x, a, b)^{\top} \theta_{h}, \qquad
        \mathbb{P}_{h}(\cdot \mid x, a, b) =\phi(x, a, b)^{\top} \mu_{h}(\cdot),
        \end{aligned}
        \end{equation}
        where $\phi:\cS \times \cA_1 \times \cA_2 \to \RR^d$ is a known feature map, $\theta_h \in \RR^d$ is an unknown vector, $\mu_h = (\mu_h^{(i)})_{i \in [d]}$ is a vector of $d$ unknown signed measure over $\cS$. We further assume that $\norm{\phi(\cdot,\cdot,\cdot)} \leq 1$, $\norm{\theta_h} \leq \sqrt{d}$, and $\norm{\mu_h(\cS)} \leq \sqrt{d}$ for all $h \in [H]$ where $\norm{\cdot}$ is the $\ell_2$-norm of vector.
        \end{assumption}
        With this assumption, we have the following result.
        \begin{lemma}[Linearity of Value Function] \label{lemma:linear:value:function}
        Under Assumption~\ref{assumption:linear:MG}, for any policy pair $(\pi, \nu)$ and any $(x, a, b, h) \in \cS \times \cA_1 \times \cA_2 \times [H]$, we have 
        \$
        Q^{\pi, \nu}(x, a, b) = \langle \phi(x, a, b), w_h^{\pi, \nu} \rangle,
        \$
        where $w_h^{\pi, \nu} = \theta_h + \int_\cS V_{h+1}^{\pi, \nu}(x') d\mu_h(x')$.
        \end{lemma}
        \begin{proof}
        The result is implied by Bellman equation in \eqref{eq:bellmam_eq} and the linearity of $r_h$ and $\PP_h$ in Assumption \ref{assumption:linear:MG}. 
        \end{proof}
        
        \subsection{Nash Equilibrium and Performance Metrics}
        If we fix some max-player's policy $\pi$, then the MG degenerates to an MDP for the min-player. By the theory of single-agent RL, we know that there exists a policy $\brt(\pi)$, referred to as the best response policy of the min-player, satisfying $V_h^{\pi,\brt(\pi)}(s) = \inf_\nu V_h^{\pi, \nu}(s)$ for all $s$ and $h$. Similarly, we define the best response policy $\brt(\nu)$ for the min-player's policy $\nu$. To simplify the notation, we define
        $$
        \begin{aligned}
        V_h^{\pi, *} &= V_h^{\pi, \brt(\pi)}, \text{ and } V_h^{*, \nu} &= V_h^{\brt(\nu),\nu}.
        \end{aligned}
        $$
         It is known that there exists a Nash equilibrium (NE) policy $(\pi^*, \nu^*)$ such that $\pi^*$ and $\nu^*$ are the best response policy to each other \citep{filar2012competitive} and we denote the value of them as $V^*_h = V_h^{\pi^*,\nu^*}$. Although multiple NE policies may exist, for zero-sum MGs, the value function is unique. 
         
         The NE policy is further known to be the solution to the following minimax equation:
         \begin{equation}
        \sup _{\pi} \inf _{\nu} V_{h}^{\pi, \nu}(s)=V_{h}^{\pi^{\star}, \nu^{\star}}(s)=\inf _{\nu} \sup _{\pi} V_{h}^{\pi, \nu}(s),  \quad \forall (s, h).       
        \end{equation}
        We also have the following weak duality property for any policy pair $(\pi, \nu)$ in MG:
        \begin{equation}
            \label{eqn:weak_dual}
            V_h^{\pi,*}(s) \leq V^*_h(s) \leq V_h^{*,\nu}(s), \quad \forall (s,h).
        \end{equation}
        Accordingly, we measure a policy pair $(\pi, \nu)$ by the duality gap:
        \begin{equation}
            \sub((\pi, \nu), x) = V_1^{*, \nu}(x) - V_1^{\pi, *}(x).
        \end{equation}
        The goal of learning is to find an $\epsilon$-approximate NE $(\hat{\pi}, \hat{\nu})$ such that $\sub((\hat{\pi}, \hat{\nu}), x) \leq \epsilon$.

      \subsection{Offline Data Collecting Process}
      We introduce the notion of compliance of dataset.
      \begin{definition}[Compliance of Dataset]
        Given an MG $\mathcal{M}$ and a dataset $\mathcal{D} = \{(s_h^{\tau},a_h^{\tau}, b_h^{\tau})\}_{\tau, h=1}^{K, H}$, we say the dataset $\mathcal{D}$ is compliant with the MG $\mathcal{M}$ if 
        \begin{align}\label{def:compliance}
          &\mathbb{P}_{\mathcal{D}}\left(r_h^{\tau} = r, s_{h+1}^{\tau} = s|\{(s_h^{i},a_h^{i}, b_h^{i})\}_{i = 1}^\tau, \{(r_h^{i},s_{h+1}^{i})\}_{i=1}^{\tau-1} \right) \notag\\
           &\qquad = \mathbb{P}_h\left(r_h = r, s_{h+1} = s| s_h = s_h^{\tau},a_h=a_h^{\tau}, b_h=b_h^\tau\right)
        \end{align} 
        for all $ h \in [H], s \in \mathcal{S}$ where $\mathbb{P}$ in the right-hand side of \eqref{def:compliance} is taken with respect to the underlying MG $\mathcal{M}$.
      \end{definition}
       We make the following assumption through this paper.
      \begin{assumption}[Date Collection]\label{assump:compliant}
        The dataset $\mathcal{D}$ is compliant with the underlying MG $\mathcal{M}$.
      \end{assumption}
     Intuitively, the compliance ensures (i) $\cD$ possesses the Markov property, and (ii) conditioned on $(s_h^\tau,a_h^\tau, b_h^\tau)$, $(r_h^\tau, s_{h+1}^\tau)$ is generated by the reward function and the transition kernel of the underlying MG. 
     
     As discussed in \citet{jin2020pessimism}, as a special case, this assumption holds if the dataset $\cD$ is collected by a fixed behavior policy. More generally, the experimenter can sequentially improve her policy by any online MARL algorithm as the assumption allows $(a_h^\tau, b_h^\tau)$ to be interdependent across the trajectories. In an extreme case, the actions can even be chosen in an adversarial manner.


     \subsection{Additional Notations} 
    For any real number $x$ and positive integer $h$, we define the regulation operation as $\mathrm{\Pi}_{h}(x) = \min\{h,\max\{x,0\}\}.$ Given a semi-definite matrix $\Lambda$, the matrix norm for any vector $v$ is denoted as $\| v \|_{\Lambda} = \sqrt{v^\top \Lambda v}$. The Frobenius norm of a matrix $A$ is given by  $\norm{A}_{F} = \sqrt{\tr(AA^\top)}$. 
    We denote $\lambda_{\text{min}}(A)$ as the smallest eigenvalue of the matrix $A$. We also use the shorthand notations $\phi_h = \phi\stateph$, $\phi_h^{\tau} = \phi\statepht$, and $r_h^\tau = r_h(s_h^\tau, a_h^\tau, b_h^\tau)$.

\section{\algname} \label{sec:algorithm}

In this section, we introduce our algorithm, namely, \algname (PMVI), whose peudocode is given in Algorithm~\ref{alg1}.

\begin{algorithm}[H]
	\caption{Pessimistic Minimax Value Iteration}
	\begin{algorithmic}[1] \label{alg1}
		\STATE Input: Dataset $\cD = \{x_h^\tau, a_h^\tau, b_h^\tau, r_h^\tau\}_{(\tau, h) \in [K] \times [H]}$.
    	\STATE Initialize $\underline{V}_{H+1}(\cdot) = \overline{V}_{H+1}(\cdot) = 0$.
    	\FOR{step $h = H, H-1, \cdots, 1$}
    	\STATE $\Lambda_h \leftarrow \sum_{\tau = 1}^K \phi_h^\tau (\phi_h^\tau)^\top + I$. \label{line:def:lambda}
		\STATE $\underline{w}_h \leftarrow \Lambda_h^{-1}(\sum_{\tau = 1}^K\phi_h^\tau (r_h^\tau + \underline{V}_{h+1}(x_{h+1}^\tau)))$.
		\STATE $\overline{w}_h \leftarrow \Lambda_h^{-1}(\sum_{\tau = 1}^K\phi_h^\tau(r_h^\tau + \overline{V}_{h+1}(x_{h+1}^\tau)))$.
		\STATE $\Gamma_h(\cdot, \cdot, \cdot) \leftarrow \beta \cdot \sqrt{\phi(\cdot, \cdot, \cdot)^\top(\Lambda_h)^{-1}\phi(\cdot, \cdot, \cdot)}$.
		\STATE $\underline{Q}_h(\cdot, \cdot, \cdot) \leftarrow \Pi_{H - h + 1}\{\phi(\cdot, \cdot, \cdot)^\top \underline{w}_h - \Gamma_h(\cdot, \cdot, \cdot)\}$.
    	\STATE ${\overline{Q}}_h(\cdot, \cdot, \cdot)\leftarrow \Pi_{H - h + 1}\{ \phi(\cdot, \cdot, \cdot)^\top \overline{w}_h + \Gamma_h(\cdot, \cdot, \cdot)\}$.
        \STATE Let $(\hat{\pi}_h(\cdot \mid \cdot), {\nu}'_h(\cdot \mid \cdot))$ be the NE of the matrix game with payoff matrix $\underline{Q}_h(\cdot, \cdot, \cdot)$.
    	\STATE Let $({\pi}'_h(\cdot \mid \cdot), \hat{\nu}_h(\cdot \mid \cdot))$ be the NE of the matrix game with payoff matrix $\overline{Q}_h(\cdot, \cdot, \cdot)$.
		\STATE $\underline{V}_h(\cdot) \leftarrow \EE_{a \sim \hat{\pi}_h(\cdot \mid \cdot) , b \sim {\nu}'_h(\cdot \mid \cdot)}\underline{Q}_h(\cdot, a, b)$.
		\STATE $\overline{V}_h(\cdot) \leftarrow \EE_{a \sim \pi'_h(\cdot \mid \cdot), b \sim \hat{\nu}_h(\cdot\mid\cdot) }\overline{Q}_h(\cdot, a, b)$.
    	\ENDFOR
    	\STATE Output: $(\hat{\pi} = \{\hat{\pi}_h\}_{h=1}^H, \hat{\nu} = \{\hat{\nu}_h\}_{h=1}^H)$.
	\end{algorithmic}
\end{algorithm}

At a high level, PMVI constructs pessimistic estimations of the value functions for both players and outputs a policy pair  based on these two estimated value functions. 

Our learning process is done through backward induction with respect to the timestep $h$. We set $\underline{V}_{H+1}(\cdot) = \overline{V}_{H+1}(\cdot) = 0$, where $\underline{V}_{H + 1}$ and $\overline{V}_{H + 1}$ are estimated value functions for max-player and min-player, respectively. Suppose we have obtained the estimated value functions $(\underline{V}_{h+1}, \overline{V}_{h+1})$ at $(h+1)$-th step, together with the linearity of value functions (Lemma~\ref{lemma:linear:value:function}), we can use the regularized least-squares regression to obtain the linear coefficients $(\underline{w}_h, \overline{w}_h)$ for the estimated Q-functions:
\$
&\underline{w}_h \leftarrow \argmin_{w} \sum_{\tau = 1}^K [r_h^\tau + \underline{V}_{h+1}(x_{h+1}^\tau) - (\phi_h^\tau)^\top w]^2 + \|w\|_2^2, \\
&\overline{w}_h \leftarrow \argmin_{w} \sum_{\tau = 1}^K [r_h^\tau + \overline{V}_{h+1}(x_{h+1}^\tau) - (\phi_h^\tau)^\top w]^2 + \|w\|_2^2,
\$
{\noindent{where}} $\phi_h^\tau$ is the shorthand of $\phi(s_h^\tau, a_h^\tau, b_h^\tau)$.
Solving this problem gives the closed-form solutions:
\begin{equation}
\begin{aligned} \label{eq:def:w}
&\underline{w}_h \leftarrow \Lambda_h^{-1}(\sum_{\tau = 1}^K\phi_h^\tau (r_h^\tau + \underline{V}_{h+1}(x_{h+1}^\tau))), \\
&\overline{w}_h \leftarrow \Lambda_h^{-1}(\sum_{\tau = 1}^K\phi_h^\tau(r_h^\tau + \overline{V}_{h+1}(x_{h+1}^\tau))), \\
& \text{where } \Lambda_h \leftarrow \sum_{\tau = 1}^K \phi_h^\tau (\phi_h^\tau)^\top + I.
\end{aligned}
\end{equation}

Unlike the online setting where optimistic estimations are essential for encouraging exploration \citep{jin2020provably,xie2020learning}, we need to adopt more robust estimation due to the distributional shift in the offline setting. Inspired by recent work \citep{jin2020pessimism,rashidinejad2021bridging,yin2021towards,uehara2021pessimistic,zanette2021provable}, which shows that \emph{pessimism} plays a key role in the offline setting, we also use the pessimistic estimations for both players. In detail, we estimate Q-functions by subtracting/adding a bonus term:
\begin{equation}
\begin{aligned} \label{eq:def:Q}
& \underline{Q}_h(\cdot, \cdot, \cdot) \leftarrow \Pi_{H - h + 1}\{\phi(\cdot, \cdot, \cdot)^\top \underline{w}_h - \Gamma_h(\cdot, \cdot, \cdot)\}, \\
& {\overline{Q}}_h(\cdot, \cdot, \cdot)\leftarrow \Pi_{H - h + 1}\{\phi(\cdot, \cdot, \cdot)^\top \overline{w}_h + \Gamma_h(\cdot, \cdot, \cdot)\}.
\end{aligned}
\end{equation}
Here $\Gamma_h$ is the bonus function, which takes the form $\beta \sqrt{\phi^\top \Lambda_h^{-1} \phi}$, where $\beta$ is a parameter which will be specified later. Such a bonus function is common in linear bandits \citep{lattimore2020bandit} and linear MDPs \citep{jin2020provably}. We remark that $\underline{Q}_h$ and $\overline{Q}_h$ are pessimistic estimations for the max-player and the min-player, respectively. Then, we solve the matrix games with payoffs $\underline{Q}_h$ and $\overline{Q}_h$:
\$
(\hat{\pi}_h(\cdot \mid \cdot), {\nu}'_h(\cdot \mid \cdot)) \leftarrow \mathrm{NE}(\underline{Q}_h(\cdot, \cdot, \cdot)), \\
({\pi}'_h(\cdot \mid \cdot), \hat{\nu}_h(\cdot \mid \cdot)) \leftarrow \mathrm{NE}(\overline{Q}_h(\cdot, \cdot, \cdot)).
\$
The estimated value functions $\underline{V}_h(\cdot)$ and $\overline{V}_h(\cdot)$ are defined by $\EE_{a \sim \hat{\pi}_h(\cdot \mid \cdot) , b \sim {\nu}'_h(\cdot \mid \cdot)}\underline{Q}_h(\cdot, a, b)$ and $\EE_{a \sim {\pi}'_h(\cdot \mid \cdot) , b \sim \hat{\nu}_h(\cdot \mid \cdot)}\overline{Q}_h(\cdot, a, b)$, respectively. After $H$ steps, PMVI outputs the policy pair $(\hat{\pi} = \{\hat{\pi}_h\}_{h=1}^H, \hat{\nu} = \{\hat{\nu}_h\}_{h=1}^H)$.

 \begin{remark}[Computational efficiency] We remark that our algorithm is computationally efficient because both the regression~\eqref{eq:def:w} and finding the NE of a zero-sum matrix game (using linear programming) can be efficiently implemented. Moreover, we remark that we do not need to compute $\overline{Q}_h(x,\cdot,\cdot),\underline{Q}_h(x,\cdot,\cdot),\hat{\pi}_h(\cdot|x), \nu'_h(\cdot|x),{\pi}'_h(\cdot|x), \hat{\nu}'_h(\cdot|x)$ for all $x \in \cS$. Instead, we only do so for the states we encounter.
\end{remark}

  \begin{remark}
        We remark that the linearity of the reward functions and the transition kernel is \textit{strictly} stronger than the linearity of value-function. In the online setting, the recent works \citep{jin2021power, huang2021towards} show that the linearity of the value function empowers \textit{statistically} efficient learning. However, we consider this stronger assumption because it is likely that it is essential for \textit{computational} efficiency due to the lack of computation tractability with general function approximation and the hardness result in \citet{du2019good} which only assumes near-linearity of value functions of MDPs (special case of MGs). 
    \end{remark}

In the following theorem, we provide the theoretical guarantees for PMVI (Algorithm~\ref{alg1}). Recall that we use the shorthand $\phi_h = \phi(s_h, a_h, b_h)$.

\begin{theorem} \label{thm:ub}
    Suppose Assumptions \ref{assumption:linear:MG} and \ref{assump:compliant} hold. Set
    $\beta = cdH\sqrt{\zeta}$ in Algorithm~\ref{alg1}, where $c$ is a sufficient large constant and $\zeta = \log(2dKH/p)$. Then for sufficient large $K$, it holds with probability $1 - p$ that
    \$
    &\sub\big((\hat{\pi}, \hat{\nu}), x\big) \\
    & \qquad \le 2\beta \sum_{h = 1}^H \EE_{\pi^*, \nu'}\left[\sqrt{\phi_h^\top \Lambda_h^{-1} \phi_h} \middle| s_1 = x\right]  + 2\beta \sum_{h = 1}^H \EE_{\pi', \nu^*}\left[\sqrt{\phi_h^\top \Lambda_h^{-1} \phi_h} \middle| s_1 = x \right] .
    \$
\end{theorem}

\begin{proof}
    See Appendix~\ref{appendix:pf:thm:ub} for a detailed proof.
\end{proof}

Theorem~\ref{thm:ub} states that the suboptimality of PMVI is upper bounded by the product of $2\beta$ and a data-dependent term, where $\beta$ comes from the the covering number of function classes and the date-dependent term will be explained in the following section.

 \section{Sufficiency: Low Relative Uncertainty} \label{sec:illustration}
In this section, we interpret Theorem~\ref{thm:ub} by characterizing the sufficient condition for achieving sample efficiency.

\subsection{Relative Uncertainty}
We first introduce the following important notion of ``relative uncertainty''.
      \begin{definition}[Relative Uncertainty]
        Given an MG $\mcm$ and a dataset $\mcd$ that is compliant with $\mcm$, for an NE policy pair $(\pi^*,\nu^*)$, the relative uncertainty of $(\pi^*,\nu^*)$ with respect to $\mcd$ is defined as
        \begin{equation}
        \begin{aligned} \label{def:related_info}
         &\ri(\mcd, \pi^*, \nu^*,x) \notag\\
         &\qquad= \max \Big\{\sup\limits_{\nu}\sum_{h=1}^H\expect_{\pi^*,\nu}\Big[\sqrt{\phi_h^\top\Lambda_h^{-1}\phi_h}\, \Big |\, s_1=x\Big],  \sup\limits_{\pi}\sum_{h=1}^H\expect_{\pi,\nu^*}\Big[\sqrt{\phi_h^\top\Lambda_h^{-1}\phi_h}\, \Big |\, s_1=x\Big]\Big\} \notag 
        \end{aligned}
        \end{equation}
        where $x$ is the initial state and expectation $\expect_{\pi^*,\nu}$ and $\EE_{\pi, \nu^*}$ are taken respect to randomness of the trajectory induced by $(\pi^*,\nu)$ and $(\pi, \nu^*)$ in the underlying MG given the fixed matrix $\Lambda_h = \sum_{\tau=1}^K \phi_h^\tau (\phi_h^\tau)^\top + I$, respectively.  
        
        We also define the relative uncertainty with respect to the dataset $\cD$ as 
        \begin{equation}
        \ri(\mcd,x) = \inf\limits_{(\pi^*,\nu^*)  \text{ is NE}}\ri(\mcd, \pi^*, \nu^*,x).
        \end{equation}
      \end{definition}
    Therefore, we can reformulate Theorem~\ref{thm:ub} as:
    \begin{equation}
        \label{eqn:rewrite_theorem}
        \sub((\hat{\pi}, \hat{\nu}), x) \leq 4\beta \cdot \ri(\cD, x).
    \end{equation}

Hence, we obtain that ``low relative uncertainty'' allows PMVI to find an approximate NE policy pair sample efficiently, which further implies that ``low relative uncertainty'' is the sufficient condition for achieving sample efficiency in offline linear MGs. 

Before we provide a detailed discussion of this notion with intuitions and examples, we first contrast our result with the single policy (optimal policy) coverage identified in the single-agent setting \citep{jin2020pessimism,xie2021bellman,rashidinejad2021bridging}.

\subsection{Single Policy (NE) Coverage is Insufficient}
As demonstrated in \citet{jin2020pessimism,xie2021bellman,rashidinejad2021bridging}, a sufficient coverage over the optimal policy is sufficient for the offline learning of MDPs. As a straightforward extension, it is natural to ask whether a sufficient coverage over the NE policy pair $(\pi^*, \nu^*)$ is sufficient and therefore minimal. However, the situation is more complicated in the MG case and we have the following impossibility result.
    \begin{proposition} Coverage of the NE policy pair $(\pi^*, \nu^*)$ is not sufficient for learning an approximate NE policy pair.
    \end{proposition}
   \begin{proof}
    We prove the result by constructing two hard instances and a dataset $\cD$ such that no algorithm can achieve small suboptimality for two instances simultaneously. We consider two simplified linear MGs $\cM_1$ and $\cM_2$ with state space $\cS = \{X\}$, action sets $\cA_1 = \{a_i: i \in [3]\}$, $\cA_2=\{b_i: i \in [3]\}$, and payoff matrices:
        \begin{equation}
        \label{eqn:game_instance}
                  R_1 = \left(
          \begin{array}{ccc}
            0.5 & {-1} & 0  \\
            {1} & \textcolor{red}{0} & {1}  \\
            0 & {-1} & 0  \\
          \end{array}\right),~~
          R_2 = \left(
          \begin{array}{ccc}
            0 & 0 & -1 \\
            1 & 0 & -1  \\
            1 & 1 & \textcolor{red}{0}  \\
          \end{array}\right).
        \end{equation}
    We consider the dataset $\cD = \{(a_2,b_2,r=0), (a_3, b_3,r=0)\}$ where the choices of action are predetermined and the rewards are sampled from the underlying game, which implies that $\cD$ is compliant with the underlying game. However, we can never distinguish these two games as they are both consistent with $\cD$. Suppose that the output policies are $\hat{\pi}(a_i) = p_i, \hat{\nu}(b_j) = q_j$ with $i,j \in [3]$, we can easily find that 
    \begin{align*}
    \sub_{\cM_1}((\hat{\pi}, \hat{\nu}), x) &= 2 - p_2 - q_2, \\
    \sub_{\cM_2}((\hat{\pi}, \hat{\nu}), x) &= p_1 + q_1 + p_2 + q_2,
    \end{align*}
    where the subscript $\cM_i$ means that the underlying MG is $\cM_i$. Therefore, we have
    $$\sub_{\cM_1}((\hat{\pi}, \hat{\nu}), x) + \sub_{\cM_2}((\hat{\pi}, \hat{\nu}), x) \geq 2,$$ 
    which implies that either $\sub\limits_{\cM_1}((\hat{\pi}, \hat{\nu}), x)$ or $\sub\limits_{\cM_2}((\hat{\pi}, \hat{\nu}), x)$ is larger than $1$.
  \end{proof}

We remark that the instances constructed in the proof also intuitively illustrate the sufficiency of the "low relative uncertainty". Suppose that the underlying MG is $\cM_1$ defined in~\eqref{eqn:game_instance} and the dataset $\cD$ now contains the information about the set of action pairs:
        \begin{equation}
        \label{eqn:action_set}
        G = \left\{(a_1,b_2),(a_2,b_2),(a_2,b_1),(a_2,b_3),(a_3,b_2)\right\}.    
        \end{equation}
Then, the learning agent has the following estimation
        \begin{equation}
        \label{eqn:game_instance_est}
          \hat{R} = \left(
          \begin{array}{ccc}
            * & {-1} & *  \\
            {1} & \textcolor{red}{0} & {1}  \\
            * & {-1} & * \\
          \end{array}\right),
        \end{equation}
    where $*$ can be arbitrary. In particular, the collected information is sufficient to verify that $(a_2, b_2)$ are best response to each other and therefore the NE policy pair.

    More generally, for the NE that is possible a mixed strategy, if we have sufficient information about $\{(\pi^*, \nu): \nu \text{ is arbitrary}\}$, we can verify that $\nu^*$ is the best response of $\pi^*$. Similarly, the information about $\{(\pi, \nu^*): \pi \text{ is arbitrary}\}$ allows us to ensure that $\pi^*$ is the best response policy to $\nu^*$. Therefore, intuitively, a sufficient coverage over these policy pairs empowers efficient offline learning of the NE.

\subsection{Interpretation of Theorem~\ref{thm:ub}}
To illustrate our theory more, we make several comments below.

\vspace{4pt}
\noindent{\textbf{Data-Dependent Performance Upper Bound.}} The upper bound in Theorem~\ref{thm:ub} is in a data-dependent manner, which is also a key idea employed by many previous works. This allows to drop the strong uniform coverage assumption, which usually fails to hold in practice. Specifically, the suboptimality guarantee only relies on the the compliance assumption and depends on the dataset $\cD$ through the relative uncertainty $\ri(\cD, x)$.

        To better illustrate the role of the relative uncertainty, we consider the linear MG $\cM_1$ constructed in  \eqref{eqn:game_instance}. We define $n_{ij}$ as the times that $(a_i, b_j)$ is taken in $\cD$. Then, we have
        $$
        \begin{aligned}
        \sup_\nu \expect_{\pi^*,\nu}\left[\sqrt{\phi_h^\top\Lambda_h^{-1}\phi_h}\middle|s_1=x\right] &= (1 + \min_j n_{2, j})^{-1/2},\\
        \sup_\pi \expect_{\pi,\nu^*}\left[\sqrt{\phi_h^\top\Lambda_h^{-1}\phi_h}\middle|s_1=x\right] &= (1 + \min_in_{i,2})^{-1/2},
        \end{aligned}
        $$
        which implies that 
        \begin{equation}
        \label{eqn:ru_example}
        \ri(\cD,x) = \ri(\mcd, \pi^*, \nu^*,x) = (1 + n^*)^{-1/2},
        \end{equation}
        where $n^* = \min_{i,j\in[3]}\{n_{2,j},n_{i,2}\}$. Hence, $\ri(\cD,x)$ measures how well the dataset $\mcd$ covers the action pairs induced by $(\pi^*,\nu)$ and $(\pi, \nu^*)$, where $\pi$ and $\nu$ are arbitrary. In particular, combining \eqref{eqn:rewrite_theorem} and \eqref{eqn:ru_example}, we obtain that
        $$\sub((\hat{\pi}, \hat{\nu}), x) \leq 4\beta \cdot (1 + n^*)^{-1/2},$$
        where we take $\beta$ as stated in the theorem. This implies that the suboptimality of Algorithm~\ref{alg1} is small if the action pair set is covered well by $\cD$, which corresponds to a large $n^*$. 
        More generally, we have the following corollary:

      \begin{corollary}[Sufficient Coverage of Relative Information]\label{cor:suff}
        Under Assumptions \ref{assumption:linear:MG} and~\ref{assump:compliant}, we assume the existence of a constant $c_1$ such that  
          \begin{equation}\label{event:suff}
          \begin{aligned}
            \Lambda_h\geqslant I + c_1 \cdot K &\cdot \max\left\{\sup\limits_{\nu}\expect_{\pi^*,\nu}\left[\phi_h\phi_h^\top\middle|s_1=x\right], \sup\limits_{\pi}\expect_{\pi,\nu^*}\left[\phi_h\phi_h^\top\middle|s_1=x\right] \right\},
          \end{aligned}
          \end{equation}
          with probability at least $1 - p/2$. Set $\beta = cdH\sqrt{\zeta}$ in Algorithm \ref{alg1} where $c$ is a sufficient large constant and $\zeta = \log(4dHK/p)$. Then for sufficient large $K$, it holds with probability $1 - p$ that  
          \$
          \sub((\hat{\pi}, \hat{\nu}), x) \leqslant c'd^{3/2}H^2K^{-1/2}\sqrt{\zeta},
          \$
          where $c'$ is a constant that only relies on $c$ and $ c_1$. 
      \end{corollary}
      \begin{proof}
        See Appendix \ref{appendix:pf:cor_suff} for detailed proof.
      \end{proof}
\noindent{\textbf{Oracle Property.}} Notably, in the above example, the action pair that lies off the set $G$ in \eqref{eqn:action_set} will not affect $\ri(\cD,x)$. Such a property is referred as the oracle property in the literature \citep{donoho1994ideal, zou2006adaptive, fan2001variable}. Specifically, since $\ri(\cD, x)$ takes expectation under the set of policy pairs:
\$
P = \{(\pi^*, \nu): \nu \text{ is arbitrary}\} \bigcup \{(\pi, \nu^*): \pi \text{ is arbitrary}\},
\$
the suboptimality automatically "adapts" to the trajectory induced by this set even though it is unknown in prior. This property is important especially when the dataset $\cD$ contains a large amount of irrelative information as the irrelative information possibly misleads other learning agents. For instance, suppose that we collect $\cD$ through a naive policy pair where both the max-player and the min-player pick their actions randomly. Therefore, all action pairs are sampled approximately uniformly. We assume that they are equally sampled for ${K}/{9}$ times for simplicity. In this case, since $n^* = {K}/{9}$, the suboptimality of Algorithm~\ref{alg1} still decays at a rate of $K^{-1/2}$. In particular, one important observation is that the output policy pair $(\hat{\pi}, \hat{\nu})$ can outperform the naive policy used to collect the dataset $\cD$.

\vspace{4pt}
 \noindent{\textbf{Well-Explored Dataset.}} As in existing literature (e.g., \cite{duan2020minimax}), we also consider the case where the data collecting process explores the state-action space well.

\begin{corollary}[Well-Explored Dataset] \label{cor_well_exp}
        Supposed the dataset $\mcd = \{(s_h^\tau,a_h^\tau,b_h^\tau,r_h^\tau)\}_{\tau,h=1}^{K,H}$ is induced by a fixed behavior policy pair $(\bpc)$ in the underlying MG. We also assume the existence of a constant $\underline{c}>0$ such that 
        \$
        \lambda_{\text{min}}(\Sigma_h) \geqslant \underline{c}\quad \text{where} \quad \Sigma_h = \expect_{\bpc}[\phi_h\phi_h^\top ],\quad \forall h \in [H].
        \$
        Set $\beta = cdH\sqrt{\zeta}$ in Algorithm \ref{alg1} where $c$ is a sufficient large constant and $\zeta = \log(4dHK/p)$. Then for sufficient large $K$, it holds with probability $1 - p$ that 
        \$
         \sub((\hat{\pi}, \hat{\nu}), x) \leqslant c'dH^2K^{-1/2}\sqrt{\zeta},
         \$
         where $c'$ is a constant that only relies on $c$ and $\underline{c}$.
\end{corollary}

\begin{proof}
     See Appendix \ref{appendix:pf:cor_well_exp} for a detailed proof.
\end{proof}

    \section{Necessity: Low Relative Uncertainty} \label{sec:minimax:optimality}

    In this section, we show that the low relative uncertainty is also the necessary condition by establishing an information-theoretic lower bound. 

    We have considered two sets of policy pairs, corresponding to two levels of coverage assumptions on the dataset:
    \begin{equation}
        \begin{aligned}
          P_1 = \{(\pi^*, \nu^*) \text{ is an NE}\}; \qquad 
          P_2 = \{(\pi^*, \nu), (\pi, \nu^*): \pi, \nu \text{ are arbitrary}\}.
        \end{aligned}
    \end{equation}
    Clearly, we have $P_1 \subset P_2$. From the discussion in Section~\ref{sec:illustration}, we know that a good coverage of $P_1$ is insufficient, while a good coverage over $P_2$ is sufficient for efficient offline learning. It remains to ask whether there is a coverage assumption weaker than $P_2$ but stronger than $P_1$ that empowers efficient offline learning in our setting. We give the negative answer by providing an information-theoretic lower bound in the following theorem.

    \begin{theorem}\label{theorem:minimax}
    For any algorithm $\algo(\cdot)$ that outputs a Markov policy pair based on $\cD$, there exists a linear game $\mathcal{M}$ and a dataset $\cD$ that is compliant with the underlying MG $\mcm$, such that when $K$ is large enough, it holds that 
  \begin{equation}\label{ineq:minimax}
    \expect_{\mathcal{D}}\left[\dfrac{\sub\left(\algo(\mathcal{D});x_{0}\right)}{\mathrm{RU} (\mcd,x_0)}\right] \geqslant C',
  \end{equation}
 where $C'$ is an absolute constant and $x_0$ is the initial state. The expectation is taken with respect to $\mathbb{P}_{\mcd}$ where $\algo(\cD)$ is a policy pair constructed based on the dataset $\cD$.
\end{theorem}
\begin{proof}
     See Appendix~\ref{sec:proof_lower} for a detailed proof.
\end{proof}

Notably, the lower bound in Theorem~\ref{theorem:minimax} matches the suboptimality upper bound in Theorem~\ref{thm:ub} up to $\beta$ and absolute constant factors and therefore establishes the near-optimality of Algorithm~\ref{alg1}. 
Meanwhile, Theorem~\ref{theorem:minimax} states that the relative uncertainty $\mathrm{RU} (\mcd,x_0)$ correctly captures the hardness of offline MG under the linear function approximation setting, that is, low relative uncertainty is the necessary condition for achieving sample efficiency.

\section{Conclusion} 
In this paper, we make the first attempt to study the two-player zero-sum linear MGs in the offline setting. For such an equilibrium finding problem, we propose a pessimism-based algorithm PMVI, which is the first RL algorithm that can achieve both computational efficiency and minimax optimality up to multiplicative factors involving the dimension and horizon. Meanwhile, we introduce a new notion of relative uncertainty and prove that low relative uncertainty is the necessary and sufficient condition for achieving sample efficiency in offline linear MGs. We believe our work opens up many promising directions for future work, such as how to perform sample-efficient equilibrium learning in the offline zero-sum MGs with general function approximations~\citep{jin2021power,huang2021towards}.

\section*{Acknowledgement}
The authors would like to thank Qiaomin Xie for helpful discussions.
 
\bibliographystyle{ims} 
\bibliography{reference.bib}
\newpage
\appendix
\onecolumn
\section{Proof of Theorem \ref{thm:ub}}  \label{appendix:pf:thm:ub}

\begin{proof}[Proof of Theorem \ref{thm:ub}]
    First, we define the Bellman error as 
    \$
    & \underline{\iota}_h(x, a, b) = \BB_h \underline{V}_{h+1}(x, a, b) - \underline{Q}_h(x, a, b), \\
    & \overline{\iota}_h(x, a, b) = \BB_h \overline{V}_{h+1}(x, a, b) - \overline{Q}_h(x, a, b).
    \$
    
	Our proof relies on the following lemma.
	\begin{lemma} \label{lemma:bellman:error}
	    Let $\cE$ denote the event that 
	    \$
		&0 \le -\overline{\iota}_h(s, a, b) \le 2 \Gamma_h(s, a, b), \\ &0 \le \underline{\iota}_h(s, a, b) \le 2 \Gamma_h(s, a, b).
		\$
		for all $h \in [H]$ and $(s, a, b) \in \cS \times \cA \times \cB$. Then we have $\Pr(\cE) \ge 1 - p$. 
	\end{lemma}
    \begin{proof}
    See Appendix~\ref{appendix:pf:lemma:bellman:error} for a detailed proof. 
    \end{proof}
    Under this event, we also have the following lemma to ensure that our estimated value functions are optimistic.
	\begin{lemma} \label{lemma:optimism}
		Under the event $\cE$, we have 
		\$
		& \underline{V}_h(x) \le V_h^{\hat{\pi}, *}(x), 
		\qquad V_h^{*, \hat{\nu}}(x) \le \overline{V}_h(x) .
		\$
	\end{lemma}
	\begin{proof}
		See Appendix~\ref{appendix:pf:lemma:optimism} for a detailed proof.
	\end{proof}
	Back to our proof, we decompose the suboptimality gap as
	\begin{equation}
	 \label{eq:330}
	\sub\big((\hat{\pi}, \hat{\nu}), x\big) = V_1^{*, \hat{\nu}}(x) - V_1^{\hat{\pi}, *}(x) = \underbrace{V_1^{*, \hat{\nu}}(x) - V_1^*(x)}_{\rm (i)} + \underbrace{V_1^*(x) - V_1^{\hat{\pi}, *}(x)}_{\rm (ii)}.
	\end{equation}
	For term (i), by Lemma~\ref{lemma:optimism}, we have
	\begin{equation}
	 \label{eq:331}
	{\rm (i)} \le \overline{V}_1(x) - V_1^*(x) 
	\le \overline{V}_1(x) - V_1^{\pi', \nu^*}(x),
	\end{equation}
	where the last inequality follows from the fact that $(\pi^*, \nu^*)$ is the NE. Then we can use the following lemma to decompose the term $\overline{V}_1(x) - V_1^{{\pi'}, \nu^*}(x)$.
	\begin{lemma}[Value Difference Lemma] \label{lemma:value:difference}
		Given an MG $(\mathcal{S},\mathcal{A}, \cB, r,H)$. Let $\hat{\pi}\otimes\hat{\nu} = \{ \hat{\pi}_h\otimes\hat{\nu}_h: \cS \rightarrow \Delta(\cA_1) \times \Delta(\cA_2)\}_{h \in [H]}$ be a product policy, $(\pi, \nu)$ be a policy pair, and $\{\qfunhat_h\}_{h=1}^H$ be any estimated $Q$-functions. For any $h \in [H]$ , we define the estimated value function $\hat{V}_h: \mathcal{S} \rightarrow \mathbb{R}$ by setting $ \hat{V}_h (x) = \langle\qfunhat_h (x,\cdot,\cdot),  \hat{\pi}_h(\cdot|x)\otimes\hat{\nu}_h(\cdot|x) \rangle$ for all $x \in \mathcal{S} $. For all $x \in \mathcal{S} $,  
          \begin{align*}
            \hat{V}_1(x)-V_1^{\pi, \nu}(x) &= \sum_{h = 1}^H \EE_{{\pi}, \nu} \Big[ \langle \hat{Q}_h(s_h, \cdot, \cdot), \hat{\pi}_h(\cdot|s_h)\otimes\hat{\nu}_h(\cdot|s_h) - {\pi}_h(\cdot|s_h) \otimes \nu_h(\cdot|s_h) \rangle | s_1 = x\Big] \\
            &\qquad + \sum_{h = 1}^H \expect_{\pi, \nu}\left[\qfunhat_h(s_h,a_h,b_h) - \bellman\hat{V}_{h+1}(s_h,a_h,b_h)|s_1 = x\right].
          \end{align*}
	\end{lemma}
	\begin{proof}
		See Section B.1 in \citet{cai2020provably} for a detailed proof. 
	\end{proof}
	By Lemma~\ref{lemma:value:difference}, we obtain
    \#
     \label{eq:332}
	\overline{V}_1(x) - V_1^{{\pi'}, \nu^*}(x)  &= \sum_{h = 1}^H \EE_{\pi', \nu^*} \Big[ \langle \overline{Q}_h(s_h, \cdot, \cdot),\pi'_h(\cdot|x)\otimes\nuhat_h(\cdot|x)- \pi'_h(\cdot|s_h) \otimes \nu_h^*(\cdot|s_h) \rangle | s_1 = x\Big]  \notag\\
  & \qquad\qquad\qquad - \sum_{h = 1}^H \EE_{\pi', \nu^*} [\overline{\iota}_h(s_h, a_h, b_h) | s_1 = x] .
  \#
	The first term can be bounded by the following lemma.
	\begin{lemma} \label{lemma:optimization:error}
		It holds that 
		\$
		\sum_{h = 1}^H \EE_{\pi', \nu^*} \Big[ \langle \overline{Q}_h(s_h, \cdot, \cdot), \pi'_h(\cdot|s_h)\otimes\hat{\nu}_h(\cdot|s_h)- \pi'_h(\cdot|s_h) \otimes \nu_h^*(\cdot|s_h) \rangle | s_1 = x\Big] \le 0.
		\$
	\end{lemma} 
	\begin{proof}
	See Appendix~\ref{proof:lemma:optimization:error} for a detailed proof.
	\end{proof}
	Applying Lemma~\ref{lemma:bellman:error} to the second term of \eqref{eq:332} gives
	\$
	- \sum_{h = 1}^H \EE_{\pi', \nu^*} [\overline{\iota}_h(s_h, a_h, b_h) | s_1 = x]  \le 2 \sum_{h = 1}^H \EE_{\pi', \nu^*} [\Gamma_h(s_h, a_h, b_h) | s_1 = x].
	\$
	Putting the above inequalities together we obtain
	\# \label{eq:333}
	{\rm (i)} &\le 2 \sum_{h = 1}^H \EE_{\pi', \nu^*} [\Gamma_h(s_h, a_h, b_h) | s_1 = x] .
	\#
	 Similarly, we can obtain 
	\# \label{eq:334}
	{\rm (ii)} \le 2 \sum_{h = 1}^H \EE_{\pi^*, \nu'} [\Gamma_h(s_h, a_h, b_h) | s_1 = x].
	\#
	Plugging \eqref{eq:333} and \eqref{eq:334} into \eqref{eq:330}, we conclude the proof of Theorem \ref{thm:ub}.
\end{proof}

      \subsection{Proof of Lemma~\ref{lemma:bellman:error}} \label{appendix:pf:lemma:bellman:error}
      \begin{proof}[Proof of Lemma~\ref{lemma:bellman:error}]
          Throughout this proof, we use the shorthands 
          \$
          \phi_h^{\tau}=\phi\statepht, \qquad \phi_h =\phi\stateph, \qquad \phi = \phi\statep.
          \$
          For the simplicity of notation, we also let
          \#
          \epsilon_h^{\tau}(V) = r_h^{\tau}+V(s_{h+1}^{\tau})-\bellman V\statepht . \label{eq:epsilon}
          \#
          By the linear MG assumption, we have $(\bellman\overline{V}_{h+1})(s,a,b) = \phi(s,a,b)^\top w_h$, where 
           \$
           w_h = \theta_h+\int_{x \in \mathcal{S}}\overline{V}_{h+1}(x)\mu_h(x)\text{d} x.
           \$
           Then we have 
          \begin{align}\label{eq:decomp_of_uncerty}
            &\abso{\phi^\top \overline{w}_h - (\bellman\overline{V}_{h+1})\statep} \notag\\
            &\qquad= \abso{\phi^\top\left(\Lambda_h^{-1} \sum_{\tau=1}^K \left(r_h^{\tau} + \overline{V}_{h+1}(s_{h+1}^{\tau})\right)\phi_h^{\tau}\right)-(\bellman\overline{V}_{h+1})\statep }  \notag\\
            &\qquad= \abso{\phi^\top\Lambda_h^{-1} \sum^K_{\tau=1} \epsilon_h^{\tau}(\overline{V}_{h+1})\phi_h^{\tau} + \phi^\top\Lambda_h^{-1}\sum_{\tau=1}^K(\bellman\overline{V}_{h+1})\statepht\phi_h^{\tau}-\phi^\top w_h} \notag \\
            &\qquad= \abso{\phi^\top\Lambda_h^{-1} \sum_{\tau=1}^K \epsilon_h^{\tau}(\overline{V}_{h+1})\phi_h^{\tau} +  \phi^\top\Lambda_h^{-1}\sum_{\tau=1}^K\phi_h^{\tau}(\phi_h^{\tau})^{\top} w_h-\phi^\top w_h  } \notag\\
            &\qquad= \abso{\phi^\top\Lambda_h^{-1} \sum_{\tau=1}^K \epsilon_h^{\tau}(\overline{V}_{h+1})\phi_h^{\tau} + \phi^\top\Lambda_h^{-1}(\Lambda_h-I) w_h-\phi^\top w_h  } \notag\\
            &\qquad= \abso{\phi^\top\Lambda_h^{-1} \sum_{\tau=1}^K \epsilon_h^{\tau}(\overline{V}_{h+1})\phi_h^{\tau}-\phi^\top\Lambda_h^{-1}w_h }  \notag \\
            &\qquad\leqslant \underbrace{\abso{\phi^\top\Lambda_h^{-1}w_h }}_{\displaystyle{\text{(i)}}}+\underbrace{\abso{\phi^\top\Lambda_h^{-1} \sum_{\tau=1}^K \epsilon_h^{\tau}(\overline{V}_{h+1})\phi_h^{\tau}}}_{\displaystyle{\text{(ii)}}}   .
          \end{align}
          Now we estimate term (i)
          \begin{align} \label{eq:est_of_i}
            \abso{ \text{(i)} } \leqslant \normof{\phi}_{\Lambda_h^{-1}}\normof{w_h}_{\Lambda_h^{-1}}
            \leqslant\normof{w_h}\normlambda{\phi}\leqslant H\sqrt{d}\normlambda{\phi},
          \end{align}
          where the second inequality follows from $\normof{\Lambda_h^{-1}}_{\text{op}} \leqslant 1$ and the third inequality follows from Lemma \ref{lemma:est_of_omega}. Here $\Vert\cdot\Vert_{\text{op}}$ denotes the operator norm of a matrix.


          Supposed that $\| V - \overline{V}_{h+1} \|_{\infty} \leqslant \epsilon$, by the definition of $\epsilon_h^{\tau}(V)$ in \eqref{eq:epsilon}, we have 
          \begin{align*}
            &\abso{\epsilon_h^{\tau}(\overline{V}_{h+1})-\epsilon_h^{\tau}(V)} \\
            &\qquad= 
          \abso{r_h^{\tau}+\overline{V}_{h+1}(s_{h+1}^{\tau})-\bellman\overline{V}_{h+1}\statepht -r_h^{\tau}-V(s_{h+1}^{\tau})+\bellman V\statepht} \notag \\ 
          &\qquad \leqslant\abso{\overline{V}_{h+1}(s_{h+1}^{\tau})-V(s_{h+1}^{\tau})} + \abso{\bellman\overline{V}_{h+1}\statepht-\bellman V\statepht}  \notag \leqslant 2\epsilon .
        \end{align*}
        Thus we have 
          \begin{align*}
            \abso{\phi^\top\Lambda_h^{-1} \sum_{\tau=1}^K \left(\epsilon_h^{\tau}(\overline{V}_{h+1})-\epsilon_h^{\tau}(V)\right)\phi_h^{\tau}} 
            &\leqslant \sum_{\tau=1}^K\abso{\phi^\top\Lambda_h^{-1} \left(\epsilon_h^{\tau}(\overline{V}_{h+1})-\epsilon_h^{\tau}(V)\right)\phi_h^{\tau}} \\
            & \leqslant \sum_{\tau=1}^K\abso{\epsilon_h^{\tau}(\overline{V}_{h+1})-\epsilon_h^{\tau}(V)}\normlambda{\phi}\normlambda{\phi_h^{\tau}} \\
            &\leqslant \sum_{\tau=1}^K\abso{\epsilon_h^{\tau}(\overline{V}_{h+1})-\epsilon_h^{\tau}(V)}\normlambda{\phi}\normof{\phi_h^{\tau}} \leqslant 2\epsilon K \normlambda{\phi},
          \end{align*} 
          where the last inequality holds since $\normof{\phi}\leqslant 1$. We define two function classes as 
\begin{equation}
\begin{aligned} \label{eq:def:function:class}
& \underline{\cQ}_{h} = \Pi_{H - h + 1}\left\{\phi(\cdot, \cdot, \cdot)^\top w - \beta \sqrt{\phi^\top \Lambda^{-1} \phi}\right\}, \\
& \overline{\cQ}_{h} = \Pi_{H - h + 1}\left\{\phi(\cdot, \cdot, \cdot)^\top w + \beta \sqrt{\phi^\top \Lambda^{-1} \phi}\right\},
\end{aligned}
\end{equation}
where the parameters $(w, \Lambda)$ satisfy $\|w\| \le H\sqrt{dK}$ and $\lambda_{\min}(\Lambda) \ge 1$. Let $\underline{\cQ}_{h, \epsilon}$ and $\overline{\cQ}_{h, \epsilon}$ be the $\epsilon$-nets of $\underline{\cQ}_{h}$ and $\overline{\cQ}_{h}$, respectively. Choose the pair $(\underline{Q}'_{h+1},\overline{Q}'_{h+1})\in \underline{\mathcal{Q}}_{h, \epsilon}\times \overline{\mathcal{Q}}_{h, \epsilon}$ such that 
          $$ \Vert \overline{Q}_{h+1}- \overline{Q}'_{h+1}\Vert_{\infty} \leqslant \epsilon,\qquad \Vert \underline{Q}_{h+1}- \underline{Q}'_{h+1}\Vert_{\infty} \leqslant \epsilon,$$
          where $\epsilon = 1/KH$. Let $V'_{h+1}(s)$ be the NE value of payoff matrix $\overline{Q}'_{h+1}(s,\cdot,\cdot)$. By Lemma \ref{lemma:nonexpand} we have
          $$\abso{V'_{h+1}(s)- \overline{V}_{h+1}(s)}\leqslant \epsilon.$$ Then we obtain 
          \begin{align} \label{eq:decomp_of_ii}
            \abso{ (\text{ii})} &= \abso{  \phi^\top\Lambda_h^{-1} \sum_{\tau=1}^K \left(\epsilon_h^{\tau}(\overline{V}_{h+1})-\epsilon_h^{\tau}(V)\right)\phi_h^{\tau} + 
            \phi^\top\Lambda_h^{-1}\sum_{\tau=1}^K\epsilon_h^{\tau}(\overline{V}'_{h+1})\phi_h^{\tau}  } \notag \\ 
            &\leqslant  \abso{ \phi^\top\Lambda_h^{-1} \sum_{\tau=1}^K \left(\epsilon_h^{\tau}(\overline{V}_{h+1})-\epsilon_h^{\tau}(\overline{V}'_{h+1})\right)\phi_h^{\tau}}+
            \abso{ \phi^\top\Lambda_h^{-1}\sum_{\tau=1}^K\epsilon_h^{\tau}(\overline{V}'_{h+1})\phi_h^{\tau}}  \notag \\
            &\leqslant 2\epsilon K\normlambda{\phi} + \abso{\phi^\top\Lambda_h^{-1}\sum_{\tau=1}^K\epsilon_h^{\tau}(\overline{V}'_{h+1})\phi_h^{\tau}}   \notag \\
            & \leqslant 2\epsilon K\normlambda{\phi} + \underbrace{\normlambda{\sum_{\tau=1}^K\epsilon_h^{\tau}(\overline{V}'_{h+1})\phi_h^{\tau}}\normlambda{\phi}}_{\displaystyle{(\text{iii})}} .
          \end{align}
          For any $\tau \in [K], h \in [H]$, we define
            \[\mathcal{F}_{h,\tau-1} := \sigma\left(\{(s_h^j,a_h^j,b_h^j)\}_{j=1}^{\min\{\tau+1,K\}} \cup \{(r_h^j,s_{h+1}^j)\}_{j=1}^\tau\right),\]
            where $\sigma(\cdot)$ is the $\sigma-$algebra generated by a set of random variables. For all $\tau \in [K]$, we have $\phi(s_h^\tau,a_h^\tau,b_h^\tau) \in \mathcal{F}_{h,\tau-1}$, as $(s_h^\tau,a_h^\tau,b_h^\tau)$ is $\mathcal{F}_{h,\tau-1}-$measurable. Besides, for any fix function $V:\mathcal{S} \rightarrow [0,H-1]$ and all $\tau \in [K]$, we have 
            \[\epsilon_h^\tau(V) = r_h^\tau+V(s_{h+1}^\tau)-(\bellman V)(s_h^\tau,a_h^\tau,b_h^\tau) \in\mathcal{F}_{h,\tau-1}\]
          and $\{\epsilon_h^\tau(V)\}_{\tau=1}^K$ is a stochastic process adapted to the filtration $\{\mathcal{F}_{h,\tau}\}_{\tau=0}^K$. By Lemma \ref{lemma:concentration}, we  obtain an estimation of term (iii). For any $\delta \in (0,1)$,
            \$\mathbb{P}\left(\normlambda{\sum_{\tau=1}^K\epsilon_h^{\tau}(V)\phi_h^{\tau}}^2\geqslant 2H^2 \log\left(\dfrac{\det(\Lambda_h)^{{1/2}}}{\delta\det(I)^{1/2}}\right) \right)\leqslant \delta .\$
          Since 
          \[\normof{\Lambda_h}_{\text{op}} = \normof{I + \sum_{\tau=1}^K \phi_h^\tau(\phi_h^\tau)^\top}_\text{op} \leqslant 1 + \sum_{\tau=1}^K\normof{\phi_h^\tau(\phi_h^\tau)^\top}_\text{op} \leqslant 1+K,\]
          we have $\det(\Lambda_h)\leqslant (1+K)^d$, which further implies   
            \begin{align*}
              \mathbb{P}&\left(\normlambda{\sum_{\tau=1}^K\epsilon_h^{\tau}(V)\phi_h^{\tau}}^2 \geqslant H^2\left(d\log(1+K)+2\log(\dfrac{1}{\delta})\right)\right) \\
               &\quad\leqslant \mathbb{P}\left(\normlambda{\sum_{\tau=1}^K\epsilon_h^{\tau}(V)\phi_h^{\tau}}^2\geqslant 2H^2 \log\left(\dfrac{\det(\Lambda_h)^{{1/2}}}{\delta\det(I)^{1/2}}\right) \right)\leqslant \delta 
            .\end{align*}
            By Lemma \ref{lemma:covering}, $\abso{\underline{\cQ}_{\epsilon, h}\times \overline{\cQ}_{\epsilon, h}}=\mathcal{N}_{h,\epsilon}^2 \leqslant \left(1+\frac{4H\sqrt{dK}}{\epsilon}\right)^{2d}\left(1+\frac{8\beta^2\sqrt{d}}{\epsilon^2}\right)^{2d^2}$. 
            Thus, by the union bound argument we have 
            \# \label{eq:est_of_iii}
            |{\rm (iii)}| \lesssim dH\sqrt{\zeta}\|\phi\|_{\Lambda_h^{-1}} 
            \#
            with probability at least $1 - p/2$.
            Combining \eqref{eq:est_of_i}, \eqref{eq:decomp_of_ii}, and \eqref{eq:est_of_iii}, we have
            \$
            \abso{\phi^\top\overline{w}_h - (\bellman\overline{V}_{h+1})\statep} \le \beta \|\phi\|_{\Lambda_h^{-1}} = \Gamma_h(s, a, b)
            \$
            with probability at least $1-p/2$. Then, we have
          \[\phi^\top\overline{w}_h +\Gamma_h(s,a,b) \ge \bellman\overline{V}_{h+1}(s,a,b) \ge -(H-h+1).\]
          The last inequality follows from $\abso{r_h}\leqslant 1$ and $\abso{\overline{V}_{h+1}(s,a,b)}\leqslant H-h$. The inequality implies 
          \[\overline{Q}_h\statep = \min\{H-h+1, \phi^\top\overline{w}_h + \Gamma_h(s,a,b)\} \le  \phi^\top \overline{w}_h + \Gamma_h(s,a,b).\] 
          Therefore, we have  
           \begin{align}\label{ineq:iota_upperb}
            \overline{\iota}_h(s_h,a_h,b_h) &=   \mathbb{B}_h\overline{V}_{h+1}(s_h,a_h,b_h)-\overline{Q}_h(s_h,a_h,b_h) \notag\\& \ge \mathbb{B}_h\overline{V}_{h+1}(s_h,a_h,b_h)- \phi^\top \overline{w}_h - \Gamma_h(s_h,a_h,b_h) \ge -2\Gamma_h(s_h,a_h,b_h).
           \end{align}
        If $\phi^\top \overline{w}_h + \Gamma_h(s,a,b) \ge H-h+1$, then, we have
        \[\overline{Q}_h\statep=\min\{H-h+1, \phi^\top \overline{w}_h + \Gamma_h(s,a,b)\} = H-h+1. \] 
        Thus, we further obtain that
        \#
        \overline{\iota}_h(s,a,b) =   \mathbb{B}_h \overline{V}_{h+1}(s, a, b) -\overline{Q}_h(s,a,b) = \mathbb{B}_h\overline{V}_{h+1}(s,a,b) - (H-h+1) \le 0.  \label{ineq:iota_lowerb}\#
        Otherwise, $\phi^\top \overline{w}_h + \Gamma_h(s,a,b) \le H-h+1$, which implies $ \overline{Q}_h(s,a,b) = \phi^\top \overline{w}_h  + \Gamma_h(s,a,b)$. In this situation, we have
        \begin{align}\label{ineq:iota_lowerbii}
           \overline{\iota}_h(s,a,b) &=   \mathbb{B}_h\overline{V}_{h+1}(s,a,b)-\overline{Q}_h(s,a,b) \notag\\ &=  \mathbb{B}_h\overline{V}_{h+1}(s,a,b)- \phi^\top \overline{w}_h - \Gamma_h(s,a,b) \le 0.
        \end{align}
        Similarly, we can prove 
        \#
        0 \le \underline{\iota}_h(s, a, b) \le 2\Gamma_h(s, a, b)
        \#
        with probability at least $1 - p/2$. Thus, the event $\cE$ happens with probability at least $1-p$, which concludes our proof.
        \end{proof}  
        
        \subsection{Proof of Lemma~\ref{lemma:optimism}} \label{appendix:pf:lemma:optimism}
        \begin{proof}[Proof of Lemma~\ref{lemma:optimism}]
		We prove the first inequality i.e., 
		$$\underline{V}_h(x) \le V_h^{\hat{\pi}, *}(x) .$$
		We prove it by induction. When $ h=H+1, V_h^{\hat{\pi},*} = \underline{V}_{h}=0,  $ the inequality holds trivially. Now we suppose the inequality holds for step $h +1$, we prove it also holds for step $h$. By definition of value function,
        \begin{align}
            V_h^{\hat{\pi},*}(x) - \underline{V}_{h}(x) &= \EE_{\hat{\pi},*}[Q_h^{\hat{\pi},*}(x,a,b)] - \EE_{\pihat, \nu'}[\underline{Q}_h(x,a,b)]  \notag \\
            & = \EE_{\hat{\pi},*}[Q_h^{\hat{\pi},*}(x,a,b)-\underline{Q}_h(x,a,b)] \notag\\
            &\qquad +\big(\EE_{\hat{\pi},*}[\underline{Q}_h(x,a,b)] - \EE_{\pihat, \nu'}[\underline{Q}_h(x,a,b)]\big). \label{3.1.1}
        \end{align}  
        By the definition that $\underline{\iota}_h(x, a, b) = \BB_h \underline{V}_{h+1}(x, a, b) - \underline{Q}_h(x, a, b)$, we have
        \#
            Q_h^{\hat{\pi},*}(x,a,b)-\underline{Q}_h(x,a,b) = \mathbb{B}_h\big(V^{\hat{\pi},*}_{h+1}(x,a,b)-\underline{V}_{h+1}(x,a,b)\big) + \underline{\iota}_h(x,a,b)\geq 0, \label{3.1.2}
        \#
        where the last inequality follows from Lemma~\ref{lemma:bellman:error} and induction assumption. Meanwhile, by the property of NE, we have
        \begin{align} \label{3.1.3}
            \EE_{\hat{\pi},*}[\underline{Q}_h(x,a,b)] - \EE_{\pihat, \nu'}[\underline{Q}_h(x,a,b)] \ge 0,
    \end{align} 
        Combining (\ref{3.1.1}), (\ref{3.1.2}) and (\ref{3.1.3}), we obtain 
        $$  V_h^{\hat{\pi},*}(x) - \underline{V}_{h}(x) \geq 0,$$
        which concludes the proof.
        \end{proof}
        
        \subsection{Proof of Lemma~\ref{lemma:optimization:error}} \label{proof:lemma:optimization:error}
        \begin{proof}[Proof of Lemma~\ref{lemma:optimization:error}]
          We estimate the each term in the summation. By the fact that $(\pi_h'(\cdot|x), \hat{\nu}_h(\cdot|x))$ is the NE of the matrix game with payoff $\overline{Q}_h(x, \cdot, \cdot)$ for any $x \in \cS$, we have
        \# \label{3.5.2}
        \langle\overline{Q}_h(s_h,\cdot,\cdot),\pi_h'(\cdot|s_h)\otimes\hat{\nu}_h(\cdot|s_h)-\pi'_h(\cdot|s_h)\otimes\nu^*_h(\cdot|s_h) \le 0. 
        \#
        Taking summation over $h \in [H]$, we obtain
        $$ \sum_{h = 1}^H \mathbb{E}_{\pi', \nu^*} \Big[ \langle \overline{Q}_h(s_h, \cdot, \cdot), \pi_h'(\cdot|s_h)\otimes\hat{\nu}_h(\cdot|s_h)- \pi'_h(\cdot|s_h) \otimes \nu_h^*(\cdot|s_h) \rangle | s_1 = x\Big] \le  0, $$
        which concludes the proof.
        \end{proof}

        \section{Proof of Corollary \ref{cor:suff}} \label{appendix:pf:cor_suff}
            \begin{proof}[Proof of Corollary \ref{cor:suff}]
              Fix $(\pi^*, \nu)$. For notational simplicity, we define 
                    \[\Sigma_h(x) = \expect_{\pi^*,\nu}[\phi\stateph\phi^\top\stateph|s_1=x].\] 
              By the assumption, we have 
              \begin{align*}
                \expect_{\pi^*,\nu}\left[\sqrt{\phi_h^\top\Lambda_h^{-1}\phi_h}\right]&\leqslant \expect_{\pi^*,\nu}\left[\sqrt{\phi_h^\top(I + c_1K\Sigma_h(x))^{-1}\phi_h}|s_1 = x\right]\\
                &=  \expect_{\pi^*,\nu}\left[\sqrt{\tr((I + c_1K\Sigma_h(x))^{-1}\phi_h\phi_h^\top)}|s_1 = x\right]\\
                &\leqslant \sqrt{\expect_{\pi^*,\nu}\left[\tr((I + c_1K\Sigma_h(x))^{-1}\phi_h\phi_h^\top)|s_1= x\right]} \\
                & =  \sqrt{[\tr((I + c_1K\Sigma_h(x))^{-1}\Sigma_h(x))]}\\
                & = \sqrt{\dfrac{1}{c_1K}} \cdot \sqrt{\tr((I + c_1K\Sigma_h(x))^{-1}(c_1K\Sigma_h(x)+I-I))}\\
                & = \sqrt{\dfrac{1}{c_1K}} \cdot \sqrt{\tr(I -(I + c_1K\Sigma_h(x))^{-1} )}  \leqslant \sqrt{\dfrac{d}{c_1K}},
              \end{align*}    
              where the second inequality follows from the Cauchy-Schwarz inequality. Thus, for any policy $\nu$, we have
              \begin{align*}
                2\beta\sum_{h=1}^H \expect_{\pi^*,\nu}[\sqrt{\phi_h^\top\Lambda_h^{-1}\phi_h}|s_1=x] \leqslant 2cc_1^{-1/2}d^{3/2}H^2K^{-1/2}\sqrt{\zeta}
              .\end{align*}
              Similarly, for any policy $\pi$, we have 
              \begin{align*}
               2\beta\sum_{h=1}^H \expect_{\pi,\nu^*}[\sqrt{\phi_h^\top\Lambda_h^{-1}\phi_h}|s_1=x] \leqslant 2cc_1^{-1/2}d^{3/2}H^2K^{-1/2}\sqrt{\zeta}
              .\end{align*}
              Let $c' = 4cc_1^{-1/2}$, by the definition of relative uncertainty, we obtain
              \[4\beta \cdot \ri(\mcd,x)\leqslant c'd^{3/2}H^2K^{-1/2}\sqrt{\zeta}.\]
              Combined with Theorem \ref{thm:ub}, we further obtain
              \[\text{SubOpt}(\text{PMVI}(\mcd),x) \leqslant c'd^{3/2}H^2K^{-1/2}\sqrt{\zeta},\]
              which concludes our proof.
            \end{proof}
   
            \section{Proof of Corollary \ref{cor_well_exp}} \label{appendix:pf:cor_well_exp}
            \begin{proof}[Proof of Corollary \ref{cor_well_exp}]
                The proof consists of two steps. In the first step, we use Lemma \ref{lemma:matrix_concertra} for concentration. In the second step, we estimate the suboptimality. Recall that we denote $\phi_h = \phi(s_h, a_h, b_h)$ , $\phi_h^\tau = \phi(s_h^\tau, a_h^\tau, b_h^\tau)$ and $\Sigma_h(x) = \expect_{\overline{\pi},\overline{\nu}}[\phi\stateph\phi^\top\stateph|s_1=x]$. Let
                \[Z_h = \sum_{\tau = 1}^K A_h^\tau  ,\quad  A_h^\tau = \phi_h^\tau (\phi_h^\tau)^\top  - \Sigma_h,\quad \forall h \in [H].\]
              Clearly, we have $\expect_{\bpc}[A_h^\tau] = 0$. Since the trajectories are induced by the behavior policy $(\bpc)$, the $K$ trajectory are i.i.d. and $\{A_h^\tau\}$ are i.i.d..  By Assumption \ref{assumption:linear:MG}, we have $\normof{\phi(s,a,b)}\leqslant 1 $ for any $(s, a, b) \in \cS \times \cA_1 \times \cA_2$, which further implies $\normop{\phi_h^\tau (\phi_h^\tau)^\top} \leqslant 1 $. Then, we have
              \[\normop{\Sigma_h} = \normop{\expect_{\bpc}[\phi_h^\tau (\phi_h^\tau)^\top]} \leqslant \expect_{\bpc}\left[\normop{\phi_h^\tau (\phi_h^\tau)^\top}\right] \leqslant 1  .\]
              Thus, we obtain 
              \$
              \normop{A_h^\tau} \leqslant \normop{\Sigma_h} + \normop{\phi_h^\tau (\phi_h^\tau)^\top} \leqslant 2, 
              \$
              and
               \begin{align*} 
                \normop{A_h^\tau(A_h^\tau)^\top} &\leqslant \normop{A_h^\tau}^2 \leqslant 4 .
               \end{align*}
               Since $\{A_h^\tau\}_{\tau \in [K}$ are i.i.d. and mean-zero, for any $h \in [H]$, we have
               \begin{align*}
               \normop{\expect_{\bpc}[Z_h^\top Z_h]}  = \normop{\sum_{\tau = 1}^K\expect_{\bpc}[A_h^\tau(A_h^\tau)^\top]} &= K \cdot \normop{\expect_{\bpc}[A_h^\tau(A_h^\tau)^\top]}\leqslant 4K 
               .\end{align*}
               Similarly, we can obtain $\normop{\expect_{\bpc}[Z_hZ_h^\top ]}  \leqslant 4K $.
               By Lemma \ref{lemma:matrix_concertra}, we have ,
               \[\mbp(\normop{Z_h}\geqslant t)\leqslant 2d\cdot\exp\left(-\dfrac{t^2/2}{4K+2t/3}\right), \quad \forall t>0.\]
               Let $t = \sqrt{10K\log\left(4dH/p\right)}$, we have
               \[\mbp(\normop{Z_h}\geqslant t ) \leqslant  \dfrac{p}{2H}.\]
               By the definition of $\Lambda_h$, we have 
               \[Z_h = \sum_{\tau=1}^K\phi_h^\tau (\phi_h^\tau)^\top   - K\Sigma_h = \Lambda_h -I-K\Sigma_h.\]
               When $K > 40/\underline{c} \cdot \log\left(4dH/p\right)$, it holds with probability at least $1 - p/2$ that
               \$
               \lambda_{\text{min}}(\Lambda_h) &= \lambda_{\text{min}}(Z_h+I+K\Sigma_h) \\
               &\geqslant K\lambda_{\text{min}}(\Sigma_h)-\normop{Z_h}\geqslant K\left(\underline{c}-\sqrt{10/K \cdot \log\left(4dH/p\right)}\right) \\
               & \geqslant K\underline{c}/2 
               \$
               for all $h\in [H]$. Let $c'' = \sqrt{2/\underline{c}}$, with probability $1-p/2$, we have 
               \[\normop{\Lambda_h^{-1}} \leqslant c''^2K^{-1}\]
               for all $h \in [H]$. Combined the fact that $\normof{\phi}(\cdot, \cdot, \cdot) \leqslant 1 $, for all $h \in [H]$, we have
               \$
              \sqrt{\phi_h^\top \Lambda_h^{-1}\phi_h} \le \normop{\Lambda_h^{-1}} \leqslant c''K^{-1/2}.
               \$ 
              Then, for any policy pair $(\pi,\nu)$, we have 
               \[\sum_{h=1}^H\expect_{\pi,\nu}\left[\sqrt{\phi_h^\top \Lambda_h^{-1}\phi_h}|s_1 = x\right] \leqslant c''HK^{-1/2} .\]
               Together with Theorem \ref{thm:ub}, we have
               \[\text{SubOpt}(\text{PMVI}(\mcd),x)\leqslant c'dH^2K^{-1/2}\sqrt{\zeta}\]
               with probability at least $1 - p$ with $c' = 4cc''$. Therefore, we finish the proof.
            \end{proof}

        \section{Proof of the Information-Theoretic Lower Bound} \label{sec:proof_lower}
        The proof is organized as follows. First, we construct a class of linear MGs $\mathfrak{M}$ and a dataset collecting process for $\cD$ which is compliant with the underlying MG. Then, given a policy pair $(\pi, \nu)$ constructed based on $\cD$, we find two hard MGs $\cM_1, \cM_2$ from the class $\mathfrak{M}$ such that the policy pair cannot achieve a desired suboptimality simultaneously. Before continuing, we introduce another notion of suboptimality, defined as
        $$\subb(({\pi}, {\nu}), x_0) = |V^*_1 - V_1^{\pi, \nu}| \leq V_1^{*,\nu} - V_1^{\pi, *} = \sub(({\pi}, {\nu}), x_0),$$
        due to the weak duality property given in \eqref{eqn:weak_dual}. Therefore, we can prove the lower bound for $\subb(({\pi}, {\nu}), x_0)$ which implies the original theorem.
        
        \subsection{Construction of the Linear MG Class $\mathfrak{M}$}
        The class $\mathfrak{M}$ is defined to be 
            $$\mathfrak{M} = \{M(p_1,p_2,p_3):p_1,p_2 \in [1/4,3/4], p_3 = \min\{p_1,p_2\}\},$$
            where $M(p_1,p_2,p_3)$ is a MG with $H\geqslant2$, state space $\mathcal{S} = \{x_0,x_1,x_2\} $ and action space $\mathcal{A}_1 =  \mathcal{A}_2 = \{y_i\}_{i=0}^{A}$ with $A = |\cA_1| \geq 3$. We fix the initial state as $x_0$. Now we define the transition kernel of the game at step $h=1$ to be
        $$
        \begin{array}{ll}
        \mathbb{P}_{1}\left(x_{1} \mid x_{0}, y_{1}, y_j\right)=p_{1} & \mathbb{P}_{1}\left(x_{2} \mid x_{0}, y_{1}, y_j\right)=1-p_{1}, ~\forall j \in \{1,\cdots,A\}, \\
        \mathbb{P}_{1}\left(x_{1} \mid x_{0}, y_{2}, y_j\right)=p_{2}, & \mathbb{P}_{1}\left(x_{2} \mid x_{0}, y_{2}, y_j\right)=1-p_{2}, ~\forall j \in \{1,\cdots,A\},\\
        \mathbb{P}_{1}\left(x_{1} \mid x_{0}, y_{i}, y_j\right)=p_{3}, & \mathbb{P}_{1}\left(x_{2} \mid x_{0}, y_{i}, y_j\right)=1-p_{3}, ~\forall i \geq 3, \forall j \in \{1,\cdots,A\}.
        \end{array}
        $$
        According to the construction, we can see that the transition is determined by the max-player's action at step $h=1$. At subsequent step $h \geq 2$, we set
            \begin{align*}
              \PP_h(x_1|x_1,y_i,y_j) = \PP_h(x_2|x_2,y_i,y_j) = 1, \forall i,j \in \{1,\cdots, A\}.
            \end{align*}
        In other words, the states $x_1$ and $x_2$ are absorbing. The reward functions of the game are defined as 
        $$
        \begin{aligned}
        &r_1(x_0, y_i, y_j) = 0, ~\forall i,j \in \{1,\cdots, A\},\\
        &r_h(x_1, y_i, y_j) = 1, r_h(x_2, y_i, y_j) = 0, ~\forall h \geq 2, \forall i,j \in \{1,\cdots, A\}.
        \end{aligned}
        $$
        We further illustrate the class $\mathfrak{M}$ in Figure~\ref{fig:mg}. To show that the game $\cM(p_1,p_2,p_3)$ is indeed a linear MG, we define the feature map $\phi\statep$ to be
             \begin{align*}
               &\phi(x_0,y_i,y_j) = (e_{(i-1)A+j},0,0) \in \mathbb{R}^{A^2+2} \quad
               \phi(x_1,y_i,y_j) = (\boldsymbol{0}_{A^2},1,0) \in \mathbb{R}^{A^2+2} \\
               &\phi(x_2,y_i,y_j) = (\boldsymbol{0}_{A^2},0,1) \in \mathbb{R}^{A^2+2} \quad 
               \forall i,j \in \{1,\cdots, A\},
             \end{align*}
        where $e_{n} \in \RR^{A^2}$ is a vector whose components are all zero except for the $n$-th one.
             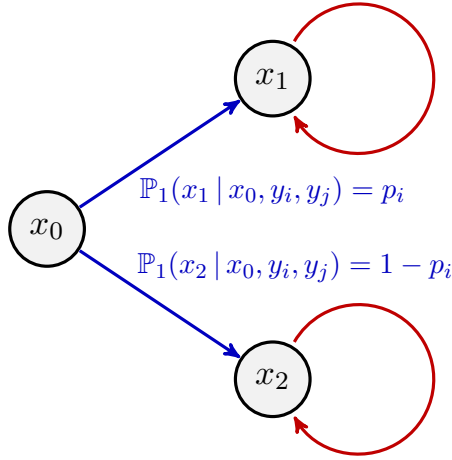
\begin{figure} [hpt]
             \label{fig:mg}
              \centering
              \begin{tikzpicture}[->,>=stealth', very thick, main node/.style={circle,draw}]
              
              \node[main node, text=black, circle, draw=black, fill=black!5, scale=1.2] (1) at  (0,0) {\small $x_{0}$};
              \node[main node, text=black, circle, draw=black, fill=black!5, scale=1.2] (2) at  (3,2) {\small $x_{1}$};
              \node[main node, text=black, circle, draw=black, fill=black!5, scale=1.2] (3) at  (3,-2) {\small $x_{2}$};

              \draw[->] (1) edge [draw=blue!75!black] (2);
              \draw[->] (1) edge [draw=blue!75!black] (3);

              \drawloop[<-,stretch=1.1]{3}{300}{420};
              \drawloop[<-,stretch=1.1]{2}{300}{420};
              \node[text=blue!75!black] at  (3,0.5) {\small $\mathbb{P}_1(x_{1}\given x_{0},y_i,y_j)=p_i$};
              \node[text=blue!75!black] at  (3.3,-0.5) {\small $\mathbb{P}_1(x_{2}\given x_{0},y_i,y_j)= 1-p_i$};
              \end{tikzpicture}
              \vspace*{-5mm}
              \caption{Illustration of the Game $\cM(p_1,p_2,p_3)$: In the first step with initial state $x_0$, the game is totally determined by the max-player's action. The game has a probability of $p_i$ to enter state $x_1$ if the max-player  takes action $a_1 = y_i$. Meanwhile, $x_1$ and $x_2$ are absorbing states.}
            \end{figure}
        \subsection{Dataset Collecting Process}\label{subsec:data}
        We specify the dataset collecting process in this subsection. Given an MG $\cM(p_1,p_2,p_3) \in \mathfrak{M}$, the dataset $\mathcal{D} = \{(s_h^{\tau},a_h^{\tau},b_h^{\tau}, r_h^{\tau})\}_{\tau, h=1}^{K, H}$ consists of $K$ trajectories starting from the initial state $x_0$, namely, $x_1^\tau = x_0$ for all $\tau \in [K]$. The actions taken at the first step $\{a_1^\tau, b_1^\tau\}_{\tau=1}^K$ are predetermined. The transitions at step $h=1$ are sampled from $\cM$ and are independent across $K$ trajectories. The rewards are also generated by the $\cM$. The subsequent actions $\{a_h^\tau, b_h^\tau\}_{\tau=1}^K, h \geq 2$ are arbitrary since they do not affect the transition and reward generation. In this case, the dataset is compliant with the underlying MG $\cM$.
        
        Before continuing, we define 
        \begin{equation}
             \begin{aligned}
               n_{ij} = \sum_{\tau=1}^K\mathbbm{1}\{a_1^\tau = y_i,b_1^\tau = y_j\}, \quad& \kappa_i^j = \sum_{\tau=1}^K\mathbbm{1}\{a_1^\tau = y_i,s_2^\tau = x_j\}, \\
               n^{k}_{\text{min}} = \min\{\min\limits_{j}n_{kj}, \min\limits_{i} n_{ik} \}, \quad n_i &= \sum_{j=1}^A n_{ij}, \quad m_i = \sum_{\tau=1}^K\mathbbm{1}\{s_2^\tau = x_i\}.
             \end{aligned}
             \end{equation}
    
    In other words, in the dataset $\cD$, the action pair $(y_i, y_j)$ is taken by two players at step $h=1$ for $n_{ij}$ times; the event that the max-player takes action $y_i$ at step $h=1$ and the next state is $x_j$ happens for $\kappa_i^j$ times; the max-player takes action $y_i$ at step $h=1$ for $n_i$ times; and the initial state $x_0$ transits to $x_i$ for $m_i$ times. Finally, $n_{\text{min}}^k$ measures how well the dataset $\cD$ covers the state action pairs where one action is fixed to be $k$.
    
    Since $x_1$ and $x_2$ are absorbing states, for learning the optimal policy $\pi^*$, the original dataset $\cD$ contains the same information as the reduced one $\cD_1 := \{(x_1^\tau, a_1^\tau, b_1^\tau, x_2^\tau, r_2^\tau)\}_{\tau=1}^K$. Recall that the actions at step $h=1$ are predetermined. The randomness of the dataset generation only comes from the transiton at the first step and we can write:
    \begin{equation}
        \label{eq:prob_dataset}
        \begin{aligned}
    \mathbb{P}_{\mathcal{D} \sim \mathcal{M}}\left(\mathcal{D}_{1}\right) &=\prod_{\tau=1}^{K} \mathbb{P}_{\mathcal{M}}\left(s_2=x_{2}^{\tau} \mid s_{1}=x_{1}^{\tau}=x_0, a_{1}=a_{1}^{\tau}, b_1 = b_1^\tau\right) \\
    &=\prod_{j=1}^{A}\left( p_j^{k_j^1}(1-p_j)^{k_j^2}\right).
\end{aligned}
    \end{equation}
    
    \subsection{Lower Bound of the Suboptimality}
    In this subsection, we constructed two MGs and show that the suboptimality of any algorithm that outputs a policy based on the dataset $\cD$ is lower bounded by the hypothesis testing risk and the risk can be further lower bounded by some tuning parameters.
    
    \begin{lemma}[Reduction to Testing] \label{lemma:reduction} For the dataset $\cD$ collected as specified in Section~\ref{subsec:data}, there exists two MGs $\cM_1(p^*, p, p), \cM_2(p, p^*, p) \in \mathfrak{M}$ where $p^* > p$ satisfy $p, p^* \in [1/4, 3/4]$, such that the output policy $\algo(\cD)$ of any algorithm satisfies:
    \begin{equation}
    \label{eq:sub_to_risk}
    \begin{aligned}
    &\expectgamei{\subb(\algo(\mathcal{D});x_{0})} +\expectgameii{\subb(\algo(\mathcal{D});x_{0})}\notag\\ &\geqslant (H-1)(p^*-p)\left(\expectgamei{1-\pihat_1(y_1)}+\expectgameii{\pihat_1(y_1)}\right) .
    \end{aligned}
    \end{equation}
    It further holds that 
            \begin{equation}
              \expectgamei{\subb(\algo(\mathcal{D});x_{0})} +\expectgameii{\subb(\algo(\mathcal{D});x_{0})} \geqslant 1/2 (H-1)(p^*-p).
            \end{equation} 
    \end{lemma}

    The right-hand side of \eqref{eq:sub_to_risk} is the risk of a (randomized) test function about the hypothesis testing problem:
    $$
    H_{0}: \mathcal{M}=\mathcal{M}_{1} \text { versus } H_{1}: \mathcal{M}=\mathcal{M}_{2}.
    $$
    This construction mirrors the Le Cam method \citep{le2012asymptotic, yu1997festschrift}. See the Section $5.3.2$ of \citet{jin2020pessimism} for a detailed discussion.
    \begin{proof}[Proof of Lemma~\ref{lemma:reduction}]
     We first notice that by the construction of $\cM_1$ and $\cM_2$, {the games degrade to the MDPs.} Therefore, we have 
     \$
     \expect_{\mathcal{D}  \sim\mathcal{M}_{\text{1}}}\left[\subb(\algo(\mathcal{D});x_{0})\right] &= \expectgamei{| V_1^{\pi^*,\nu^*}(x_0) - V_1^{\pihat,\nuhat}(x_0)|} \\
     & =\expectgamei{V_1^{\pi^*,\nu^*}(x_0) - V_1^{\pihat,\nuhat}(x_0)},
     \$
     where we use the fact that the Nash value is the V-value of the induced MDP in the last equality. Clearly, for $\cM_1$, $\pi_1^*$ puts probability $1$ for action $y_1$ given the state $x_0$. In this case, we have the following calculation:
     $$
     \begin{aligned}
     \expectgamei{V_1^{\pi^*,\nu^*}(x_0) - V_1^{\pihat,\nuhat}(x_0)} &= (H-1)\left(p^*-\sum_{i=1}^{A}\expectgamei{\pihat_1(y_i)}p_i\right) \\
     &= (H-1)\Eb_{\cD \sim \cM_1} \left( p^*(1-\hat{\pi}_1(y_1)) - \sum_{i \neq 1} \hat{\pi}_i(y_i) p \right)\\
     &= (H-1)(p^*-p)\expectgamei{1-\pihat_1(y_1)},
     \end{aligned}
     $$
    where we use $\sum_{j=1}^A \hat{\pi}_1(y_j) = 1$ in the last equality. Therefore, we have
        $$\expectgamei{\subb(\algo(\mathcal{D});x_{0})} = (H-1)(p^*-p)\expectgamei{1-\pihat_1(y_1)}.$$
        Similarly,
        $$\expectgameii{\subb(\algo(\mathcal{D});x_{0})} = (H-1)(p^*-p)\expectgameii{1-\pihat_1(y_2)}.$$
    It follows that 
    \begin{equation}
    \label{eq:add_gap}
        \begin{aligned}
          &\expectgamei{\subb(\algo(\mathcal{D});x_{0})}+\expectgameii{\subb(\algo(\mathcal{D});x_{0})}\\ 
          &\qquad \geqslant (H-1)(p^*-p)\left(\expectgamei{1-\pihat_1(y_1)}+\expectgameii{\pihat_1(y_1)}\right),
        \end{aligned}
    \end{equation}
    where we use $\hat{\pi}_1(y_1) \leq 1 - \hat{\pi}_1(y_2)$. This concludes the proof of \eqref{eq:sub_to_risk}. It remains to find a lower bound for the right-hand side. We have:
        \begin{align}\label{eq:kl_est}
          \expectgamei{1-\pihat_1(y_1)}+\expectgameii{\pihat_1(y_1)} &\geqslant 1- \text{TV}\left(\mathbb{P}_{\mathcal{D} \sim \mathcal{M}_1}, \mathbb{P}_{\mathcal{D} \sim \mathcal{M}_2}\right)\notag \\
           &\geqslant 1- \sqrt{\text{KL}\left(\mathbb{P}_{\mathcal{D} \sim \mathcal{M}_1}|| \mathbb{P}_{\mathcal{D} \sim \mathcal{M}_2}\right)/2}
        \end{align}
    where $\text{TV}(\cdot,\cdot)$ and $\text{KL}(\cdot||\cdot)$ are the total variation distance of probability measures and Kullback-Leibler (KL) divergence of two distributions, respectively. Here the first inequality comes from the definition of total variation distance, and the last inequality follows from Pinsker's inequality. {Intuitively, we set $p$ and $p^*$ carefully to make $\cM_1$ and $\cM_2$ hard to be distinguished.}
 
    As stated in \eqref{eq:prob_dataset}, we can explicitly write down the probability of the reduced dataset $\cD_1$ as
              \begin{align*}
            \mathbb{P}_{\mathcal{D}\sim\mcm_1}\left(\mcd_1\right) = (p^*)^{\kappa_1^1}\cdot(1-p^*)^{\kappa^2_1}\cdot p^{\sum_{i\neq 1}\kappa_i^1}\cdot(1-p)^{\sum_{i\neq 1}\kappa_i^2};\\
            \mathbb{P}_{\mathcal{D}\sim\mcm_2}\left(\mcd_1\right) = (p^*)^{\kappa_2^1}\cdot(1-p^*)^{\kappa_2^2}\cdot p^{\sum_{i\neq 2}\kappa_i^1}\cdot(1-p)^{\sum_{i\neq 2}\kappa_i^2}.
          \end{align*}
    We recall $\kappa_i^j = \sum_{\tau=1}^K\mathbbm{1}\{a_1^\tau = y_i,s_2^\tau = x_j\}$. Since the randomness only comes from the state transition at the first step for $\cD_1$ and these transitions are independent across $K$ trajectories. It follows that 
              \begin{align}
            &\text{KL}\left(\mathbb{P}_{\mathcal{D} \sim \mathcal{M}_1}|| \mathbb{P}_{\mathcal{D} \sim \mathcal{M}_2}\right) \notag\\
            &\qquad= \expect_{\mcd\sim\mcm_1}\bigg[\left(\kappa_1^1-\kappa_2^1\right)\log\left(\dfrac{p^*}{p}\right)+\left(\kappa_1^2-\kappa_2^2\right)\log\left(\dfrac{1-p^*}{1-p}\right)\bigg] \notag\\
            &\qquad= \left(p^*n_1-pn_2\right)\log\left(\dfrac{p^*}{p} \right) +\left((1-p^*)n_1-(1-p)n_2\right) \log\left(\dfrac{1-p^*}{1-p}\right) \notag\\
            &\qquad= n_1\left(p^*\log\dfrac{p^*}{p}+(1-p^*)\log\dfrac{1-p^*}{1-p}\right) +n_2\left(p\log\dfrac{p}{p^*}+(1-p)\log\dfrac{1-p}{1-p^*}\right).
          \end{align}
    It remains to carefully set $p$ and $p^*$ to obtain the desired lower bound. To this end, we set 
    $$p = \dfrac{1}{2} - \dfrac{1}{16}\sqrt{\dfrac{2}{n_1+n_2}}, \quad p^* = \dfrac{1}{2} + \dfrac{1}{16}\sqrt{\dfrac{2}{n_1+n_2}},$$
    such that $p, p^* \in [1/4, 3/4]$ and 
    $$p^*-p <\dfrac{1}{4}<\text{min}\{p,p^*,1-p^*,1-p\}.$$
          As a result of the inequality $\log(1+x)\leqslant x ,\forall x>-1$, we have
          $$
          \begin{aligned}
          p^*\log\dfrac{p^*}{p}+(1-p^*)\log\dfrac{1-p^*}{1-p} \leqslant \dfrac{(p^*-p)^2}{p(1-p)}\leqslant 16(p^*-p)^2, \\
          p\log\dfrac{p}{p^*}+(1-p)\log\dfrac{1-p}{1-p^*} \leqslant \dfrac{(p^*-p)^2}{p^*(1-p^*)}\leqslant 16(p^*-p)^2.
          \end{aligned}
        $$
          Thus,
          \begin{align*}
            \text{KL}\left(\mathbb{P}_{\mathcal{D} \sim \mathcal{M}_1}|| \mathbb{P}_{\mathcal{D} \sim \mathcal{M}_2}\right) \leqslant 16n_1(p^*-p)^2+16n_2(p^*-p)^2 \leqslant16(n_1+n_2)(p^*-p)^2 \leq \dfrac{1}{2}. \\
          \end{align*}
         It follows that
            \begin{align*}
              \expectgamei{1-\pihat_1(y_1)}+\expectgameii{\pihat_1(y_1)} \geqslant  1- \sqrt{\text{KL}\left(\mathbb{P}_{\mathcal{D} \sim \mathcal{M}_1}|| \mathbb{P}_{\mathcal{D} \sim \mathcal{M}_2}\right)/2} \geqslant \dfrac{1}{2}.
            \end{align*}
          Combined this with \eqref{eq:add_gap} and \eqref{eq:kl_est}, we conclude that 
          \begin{equation}\label{eq:lowerbound}
            \expectgamei{\subb(\algo(\mathcal{D});x_{0})} +\expectgameii{\subb(\algo(\mathcal{D});x_{0})} \geqslant \frac{1}{2}(H-1)(p^*-p).
          \end{equation}
         Therefore, we conclude the proof.
    \end{proof}

    \subsection{Upper Bound of $\ri(\cD, x_0)$}
    We recall that we are concerning about 
    $$    \expect_{\mathcal{D}\sim\mathcal{M}}\left[\dfrac{\subb\left(\text{Algo}(\mathcal{D});x_{0}\right)}{\ri(\mcd,x_0)}\right].$$
    We still need to find an upper bound of $\ri(\cD,x_0)$ for the constructed linear MGs.
    
    \begin{lemma}[Upper Bound of $\ri(\cD, x_0)$] \label{lemma:upper_ru}
    Suppose the Assumption~\ref{assump:compliant} holds and the underlying MG is $\cM \in \mathfrak{M}$. We define $j^* = \argmax_{j \in \{1,2\}} p_j$ (we assume that $p_1 \neq p_2$). Then, the optimal policy satisfies $\pi_1^*(y_{j^*}) = 1$ and we further take $\nu^*_1(y_{j^*}) = 1$. Then, for Algorithm~\ref{alg1}, it holds that
\begin{equation}
\label{eqn:upper_ru1}
\begin{aligned}
&\sum_{h=1}^{H}  \sup_{\nu} \mathbb{E}_{\pi^{*}, \nu}\left[\left(\phi\left(s_{h}, a_{h}, b_h\right)^{\top} \Lambda_{h}^{-1} \phi\left(s_{h}, a_{h}, b_h\right)\right)^{1 / 2} \mid x_{0}\right] \\
& \qquad \leq \frac{1}{\sqrt{1+n^{j^*}_{ \text{min}}}} +(H-1) \cdot\left(\frac{p_{j^{*}}}{\sqrt{1+m_{1}}}+\frac{1-p_{j^{*}}}{\sqrt{1+m_{2}}}\right),\\
&\sum_{h=1}^{H}  \sup_{\pi} \mathbb{E}_{\pi, \nu^*}\left[\left(\phi\left(s_{h}, a_{h}, b_h\right)^{\top} \Lambda_{h}^{-1} \phi\left(s_{h}, a_{h}, b_h\right)\right)^{1 / 2} \mid x_{0}\right] \\
&\qquad \leq \sup_\pi \left\{\frac{1}{\sqrt{1+n^{j^*}_{ \text{min}}}} + \sum_{i \in [A]}\pi_1(y_{i})(H-1) \cdot\left(\frac{p_{i}}{\sqrt{1+m_{1}}}+\frac{1-p_{i}}{\sqrt{1+m_{2}}}\right)\right\}.
\end{aligned}
\end{equation}
Furthermore, with probability at least $1-\frac{1}{K}$, the following event holds 
$$
\overline{\mathcal{E}} = \{m_i\geqslant K/4-\sqrt{2K\log(2K)} \mid i = 1,2 \},$$
where the probability is taken with respect to $\PP_{\cD \sim \cM}$. Under $\overline{\mathcal{E}}$, for $K\geqslant 32\log(2K)$, we can obtain 
\begin{equation}
\label{eqn:upper_ru}
\mathop{\mathrm{RU}}(\cD, \pi^*, \nu^*, x_0) \leq \frac{1}{\sqrt{n^{j^*}_{\text{min}}}} + \frac{2\sqrt{2}(H-1)}{\sqrt{K}} \leq \frac{2\sqrt{2} H}{\sqrt{n^{j^*}_{\text{min}}}}.
\end{equation}
\end{lemma}
    \begin{proof}[Proof of Lemma \ref{lemma:upper_ru}]\textbf{Proof of \eqref{eqn:upper_ru1}.} We first consider $\sup_\nu \Eb_{\pi^*, \nu}$. By $x_1^\tau = x_0$ for all $\tau \in [K]$ and the definition of $\Lambda_h$, we have 
     \$
\Lambda_{1} &= I+\sum_{\tau=1}^{K} \phi\left(x_{0}, a_{1}^{\tau}, b_{1}^{\tau}\right) \phi\left(x_{0}, a_{1}^{\tau}, b_{1}^{\tau}\right)^{\top} \\
&=\operatorname{diag}\left(1+n_{11}, 1 + n_{12} \ldots, 1+n_{AA}, 1, 1\right) \in \mathbb{R}^{(A^2+2) \times(A^2+2)},
\$
where the second equality follows from the definition of $\phi$. For $h \geq 2$, the state is $x_1$ or $x_2$, so we have
     $$
\Lambda_{h}=I+\sum_{\tau=1}^{K} \phi\left(x_h^\tau, a_{h}^{\tau}, b_{h}^{\tau}\right) \phi\left(x_h^\tau, a_{h}^{\tau}, b_{h}^{\tau}\right)^{\top}=\operatorname{diag}\left(1, 1, \ldots, 1, 1 + m_1, 1 + m_2\right) \in \mathbb{R}^{(A^2+2) \times(A^2+2)},
$$
where the second equality follows from the definition of $\phi$. Under $(\pi^*, \nu)$, we know that 
$$\PP_{\pi^*, \nu}(s_2=x_1) = p_{j^*}, ~~~ \PP_{\pi^*, \nu}(s_2=x_2) = 1-p_{j^*}.$$

It follows that 
\begin{equation*}
\begin{aligned}
&\sup_\nu \mathbb{E}_{\pi^{*}, \nu}\left[\left(\phi\left(s_{h}, a_{h}, b_h\right)^{\top} \Lambda_{h}^{-1} \phi\left(s_{h}, a_{h}, b_h\right)\right)^{1 / 2} \mid s_{1}=x_{0}\right] \\
&\qquad \leq \begin{cases}\left(1+n^{j^{*}}_{\text{min}}\right)^{-1 / 2}, & h=1, \\
p_{j^{*}} \cdot\left(1+m_{1}\right)^{-1 / 2}+\left(1-p_{j^{*}}\right) \cdot\left(1+m_{2}\right)^{-1 / 2}, & h \in\{2, \ldots, H\},
\end{cases}
\end{aligned}
\end{equation*}
where we use the definition of $\phi$, and the fact that $n_{\text{min}}^{j^*} \leq n_{j^* i}$ for all $i \in [A]$.

For $\sup_{\pi} \Eb_{\pi, \nu^*}$, the main difference lies in the distribution of $s_2$:
$$\PP_{\pi, \nu^*}(s_2=x_1) = \sum_{j \in [A]} \pi_1(y_j)p_j, ~~~ \PP_{\pi, \nu^*}(s_2=x_2) = 1 - \sum_{j \in [A]} \pi_1(y_j)p_j.$$
It follows that 
\begin{equation*}
\begin{aligned}
&\mathbb{E}_{\pi, \nu^*}\left[\left(\phi\left(s_{h}, a_{h}, b_h\right)^{\top} \Lambda_{h}^{-1} \phi\left(s_{h}, a_{h}, b_h\right)\right)^{1 / 2} \mid s_{1}=x_{0}\right] \\
&\qquad \leq \begin{cases}\left(1+n^{j^{*}}_{\text{min}}\right)^{-1 / 2}, & h=1, \\
\sum_{i \in [A]}\pi_1(y_{i})(H-1) \cdot\left(\frac{p_{i}}{\sqrt{1+m_{1}}}+\frac{1-p_{i}}{\sqrt{1+m_{2}}}\right), & h \in\{2, \ldots, H\},\end{cases}
\end{aligned}
\end{equation*}
where we use the definition of $\phi$, and the fact that $n_{\text{min}}^{j^*} \leq n_{i j^*}$ for all $i \in [A]$. This concludes the proof of \eqref{eqn:upper_ru1}.

We now turn to the high-probability event: 
$$
\overline{\mathcal{E}} = \left\{m_i\geqslant K/4-\sqrt{\frac{1}{2}K\log(2K)} \mid i = 1,2 \right\}.$$
By construction, we know that $\frac{3}{4} \geq p_1, p_2 \geq \frac{1}{4}$. Therefore, we know that $\expect[m_i] \geqslant 1/4K$ for $i=1,2$. By Hoeffding's inequality, for any $\xi \in (0,1) $, with probability at least $1- \xi$ , the following event happens
$$
\left\{m_i\geqslant K/4-\sqrt{\frac{1}{2}K\log(2/\xi)} \mid i = 1,2 \right\}.
$$
Setting $\xi = 1/K$, we obtain the desired result.

\noindent\textbf{Proof of \eqref{eqn:upper_ru}.} In particular, for  $K\geqslant 32\log(2K)$, with probability at least $1-1/K$, we have 
$$m_i \geqslant K/8, ~~\forall i \in \{1,2\}.$$ 
The \eqref{eqn:upper_ru} follows directly from \eqref{eqn:upper_ru1} and $m_i \geqslant K/8$.
\end{proof}

        \subsection{Proof of Theorem \ref{theorem:minimax}}
        We now invoke Lemma~\ref{lemma:reduction} and Lemma~\ref{lemma:upper_ru} to give a detailed proof of Theorem \ref{theorem:minimax}.
        
        \begin{proof}[Proof of Theorem \ref{theorem:minimax}]\label{proof_of_minimax}
            Since the actions are predetermined, we can additionally assume that 
            $$ \dfrac{1}{\overline{c}} \leqslant \dfrac{n^1_{\min}}{n^2_{\text{min}}}\leqslant\overline{c}, \quad \dfrac{1}{\overline{c}} \leqslant \dfrac{n_{2i}}{n^2_{\text{min}}}\leqslant\overline{c},\ \ \dfrac{1}{\overline{c}} \leqslant \dfrac{n_{1i}}{n^2_{\text{min}}}\leqslant\overline{c} \quad \forall i \in \{1,\cdots,n\},$$
             where $\overline{c}$ is a positive constant. {This assumption means that the numbers of action pairs whose components contain $y_1$ or $y_2$ are relatively uniform.} By Lemma \ref{lemma:reduction}, there exist two games $\mathcal{M}_1,\mathcal{M}_2$ such that 
              \begin{align*}
                \max\limits_{i\in\{1,2\}}&\sqrt{n_{\text{min}}^i}\expect_{\mathcal{D} \sim \mathcal{M}_i}\left[\subb\left(\algo(\mathcal{D}), x_{0}\right)\right] \\
                &\geqslant\dfrac{\sqrt{n_{\text{min}}^1 n_{\text{min}}^2}}{\sqrt{n_{\text{min}}^1}+\sqrt{n_{\text{min}}^2}}\left(\expectgamei{\subb\left(\algo(\mathcal{D}), x_{0}\right)}+\expectgameii{\subb\left(\algo(\mathcal{D}), x_{0}\right)}\right) \\&\geqslant \dfrac{\sqrt{n_{\text{min}}^1 n_{\text{min}}^2}}{\sqrt{n_{\text{min}}^1}+\sqrt{n_{\text{min}}^2}} \frac{1}{2}(H-1)(p^*-p) = \dfrac{\sqrt{2}}{16}\cdot (H-1) \cdot \dfrac{\sqrt{n_{\text{min}}^1 n_{\text{min}}^2}}{\sqrt{n_{\text{min}}^1}+\sqrt{n_{\text{min}}^2}} \cdot \dfrac{1}{\sqrt{n_1+n_2}},
              \end{align*}
             where the first inequality is because $\max\{x, y\} \geq a x + (1-a) y$ for all $a \in [0,1]$ and $x, y \geq 0$ and the second inequality follows from Lemma~\ref{lemma:reduction}.
            Note that
                \$n_1+n_2 = \sum_{i=1}^A(n_{1i}+n_{2i})\leqslant2\overline{c}An^{2}_{\text{min}}.\$
             Therefore, we have
             \begin{equation}
                \label{eqn:lower_subopt}
              \begin{aligned}
                \max\limits_{i\in\{1,2\}}\sqrt{n_{\text{min}}^i}\expect_{\mathcal{D} \sim \mathcal{M}_i}\left[\subb\left(\algo(\mathcal{D}), x_{0}\right)\right]
                \geqslant \dfrac{1}{16\sqrt{\overline{c}A}}\cdot (H-1)\cdot \dfrac{\sqrt{{n^{1}_{\text{min}}}/{n^{2}_{\text{min}}}}}{\sqrt{{n^{1}_{\text{min}}}/{n^{2}_{\text{min}}}}+1}\geqslant C,
              \end{aligned}
                \end{equation}
              where $C = \dfrac{1}{16\sqrt{\overline{c}A}}\cdot (H-1)\cdot\dfrac{1}{1+\sqrt{\overline{c}}}$. Here the last inequality is because $f(t) = \frac{t}{1+t}$ is increasing for $t\geq 0$. We now take the optimal policy of game $\mathcal{M}_i$ to be $\pi^*_1(y_i) = \nu^*_1(y_i) = 1$. It follows that 
              \begin{align*}
                &\max\limits_{i\in \{1,2\}}\expect_{\mathcal{D} \sim \mathcal{M}_i}\left[ \dfrac{\subb\left(\algo(\mathcal{D}), x_{0}\right)}{\ri(\mcd,x_0)}\right]\\
                &\qquad \geqslant\max\limits_{i\in \{1,2\}}\expect_{\mathcal{D} \sim \mathcal{M}_i}\left[ \dfrac{\subb\left(\algo(\mathcal{D}), x_{0}\right)}{\ri(\mcd,x_0)}\mathbbm{1}_{\overline{\mathcal{E}}}\right]\\
                &\qquad \geqslant\max\limits_{i\in \{1,2\}}\expect_{\mathcal{D} \sim \mathcal{M}_i}\left[\dfrac{1}{C_1}\sqrt{n_{\text{min}}^i} \subb(\algo(\cD), x_0)\mathbbm{1}_{\overline{\mathcal{E}}}\right]\\
                &\qquad = \max\limits_{i\in \{1,2\}}\expect_{\mathcal{D} \sim \mathcal{M}_i}\left[\dfrac{1}{C_1}\sqrt{n_{\text{min}}^i}\subb\left(\algo(\mathcal{D}), x_{0}\right)\right]\\
                &\qquad \qquad - \expect_{\mathcal{D} \sim \mathcal{M}_i}\left[\dfrac{1}{C_1}\sqrt{n_{\text{min}}^i}\subb\left(\algo(\mathcal{D}), x_{0}\right)\cdot\mathbbm{1}_{\overline{\mathcal{E}}^c}\right]\\
                &\qquad \geqslant \dfrac{C}{C_1} - \dfrac{\sqrt{K}}{C_1}\cdot2H\cdot\dfrac{1}{K}\\
                &\qquad \geqslant\dfrac{C}{2C_1} :=  C' >0,
              \end{align*} 
              where $C' =  \dfrac{C}{2C_1} , C_1 = 2\sqrt{2}H $, and $C = \dfrac{1}{16\sqrt{\overline{c}A}}\cdot (H-1)\cdot\dfrac{1}{1+\sqrt{\overline{c}}}$. The second inequality follows from \eqref{eqn:upper_ru}. The third inequality is because \eqref{eqn:lower_subopt}, $\subb(\algo(\cD), x_0) \leq 2H$, $n_{\text{min}}^i \leq K$, and $\PP(\overline{\cE}) \leq \frac{1}{K}$. The forth inequality holds for $K \geqslant \dfrac{16H^2}{C^2}$. By $\sub(\algo(\cD), x_0) \geq \subb(\algo(\cD), x_0)$, we conclude the proof of Theorem~\ref{theorem:minimax}.
        \end{proof}

        \section{Technical Lemmas}
        Recall that we use shorthands
        $$ \phi_h = \phi(s_h,a_h,b_h),\qquad\phi_h^\tau = \phi(s_h^\tau,a_h^\tau,b_h^\tau),\qquad r_h^\tau = r(s_h^\tau,a_h^\tau,b_h^\tau) .$$
        \begin{lemma} \label{lemma:est_of_omega}
          For any dataset $\mcd$, we define 
          $$w_h = \theta_h + \int_{x\in\mathcal{S}}\overline{V}_{h+1}(x)\mu_h(x)\d x,$$ 
          where $\overline{V}_{h+1}(x)$ are the value function constructed in Algorithm \ref{alg1}, then $w_h$ as well as $\overline{w}_h$  in Algorithm \ref{alg1} satisfy 
          \[ \normof{w_h} \leqslant H\sqrt{d}, \quad \normof{\overline{w}_h}\leqslant H\sqrt{Kd}, \quad i = 1,2.\]
        \end{lemma}
        \begin{proof}[Proof of Lemma \ref{lemma:est_of_omega}] By definition of $w_h$,
          \begin{align*}
            \normof{w_h} &= \normof{\theta_h + \int_{x\in\mathcal{S}}\overline{V}_{h+1}(x)\mu_h(x)\d x}\leqslant \normof{\theta_h} + \normof{\int_{x\in\mathcal{S}}\overline{V}_{h+1}(x)\mu_h(x)\d x}\\
            &\leqslant \sqrt{d} + (H-h)\int_{x\in\mathcal{S}} \normof{\mu_h(x)}\d x \leqslant \sqrt{d} + (H-h)\sqrt{d} \leqslant H\sqrt{d}
          \end{align*}
          where the second and the last inequities follow from the regulation assumption in Assumption~\ref{assumption:linear:MG} that  $\normof{\theta_h} \leqslant \sqrt{d}$ and $ \int_{x\in\mathcal{S}} \normof{\mu_h(x)}\d x \leqslant \sqrt{d}$ , while the construction in Algorithm \ref{alg1}  guarantees $\overline{V}_{h+1}(x)\leqslant H-h$, which implies the third inequality.

          By construction of $\overline{w}_h$ in Algorithm \ref{alg1}, we have
          \begin{align*}
            \normof{\overline{w}_h} = \normof{\Lambda_h^{-1}\sum_{\tau=1}^K\phi_h^\tau(r_h^\tau+\overline{V}_{h+1}(s_{h+1}))}\leqslant H\sum_{\tau=1}^K\normof{\Lambda_h^{-1}\phi_h^\tau},
          \end{align*}
          where the last inequality follows from triangle inequality and $\abso{r_h^\tau} \leqslant 1, \overline{V}_{h+1}(x)\leqslant H-h $. Note that 
          \[\normof{\Lambda_h^{-1}\phi_h^\tau} = \sqrt{(\phi_h^\tau)^\top\Lambda_h^{-1/2}\Lambda_h^{-1}\Lambda_h^{-1/2}\phi_h^\tau} \leqslant \left((\phi_h^\tau)^\top\Lambda_h^{-1}\phi_h^\tau\right)^{1/2}.\]
          The last inequality follows from $\normof{\Lambda_h^{-1}}_{\text{op}}\leqslant 1$. Thus,
          \begin{align*}
            H\sum_{\tau=1}^K\normof{\Lambda_h^{-1}\phi_h^\tau} &={H}\sum_{\tau=1}^K\left((\phi_h^\tau)^\top\Lambda_h^{-1}\phi_h^\tau\right)^{1/2} \leqslant H\sqrt{K}\left(\sum_{\tau=1}^K(\phi_h^\tau)^\top\Lambda_h^{-1}\phi_h^\tau\right)^{1/2} \\
            &= H\sqrt{K}\left(\tr(\Lambda_h^{-1}\sum_{\tau=1}^K\phi_h^\tau(\phi_h^\tau)^\top)\right)^{1/2} = H\sqrt{K}\left(\tr(\Lambda_h^{-1}(\Lambda_h-1 I))\right)^{1/2}\\
            &\leqslant H\sqrt{K}\left(\tr(I)\right)^{1/2} = H\sqrt{Kd},
          \end{align*}
          where the first inequality follows from Cauchy-Schwarz inequality.
        \end{proof}
        \begin{lemma}[Non-Expansive Property of Nash Value]\label{lemma:nonexpand}
            For any integer $n$ and matrix $A \in \RR^{n \times n}$, we denote $f(A) = \max_{x\in \Delta}\min_{y\in\Delta} x^\T A y $, where $\Delta = \{x\in \mathbb{R}^n: x_i \ge 0 , \sum_{i=1}^n x_i = 1 \}$. Fix $\epsilon > 0$, given any matrices $A_1, A_2 \in \mathbb{R}^{n\times n}$ satisfying $\normof{A_1-A_2}_\infty \leq \epsilon$, we have 
           \$
           \abso{f(A_1)-f(A_2)} \leq \epsilon.
           \$
        \end{lemma}
        \begin{proof} Fix $x$, we have
          \begin{align*}
              \min_{y\in \Delta} x^\T A_1 y& = \min_{y\in \Delta} x^\T (A_1-A_2+A_2) y  \geq  \min_{y\in \Delta} x^\T A_2 y -\epsilon,
          \end{align*}
          where the inequality follows from the fact that $\normof{A_1-A_2}_\infty \leq \epsilon$. Hence, we can further obtain 
          \begin{align}
              f(A_1) = \max_{x \in \Delta}\min_{y\in \Delta} x^\T A_1 y \ge \max_{x \in \Delta}  \min_{y\in \Delta} x^\T A_2 y -\epsilon = f(A_2) -\epsilon .
          \end{align}
          Symmetrically, we can obtain $ f(A_2) \geq f(A_1) -\epsilon $. Therefore, we conclude the proof.
        \end{proof}

        \begin{lemma}[$\epsilon-$Covering]\label{lemma:covering}
          For any $\epsilon>0$, the $\epsilon-$covering number $\mathcal{N}_{h,\epsilon}$ of $\overline{\mathcal{Q}}_{h}$ (and $ \underline{\mathcal{Q}}_{h}$) with respect to $\ell_{\infty} $ norm satisfies
            \[\mathcal{N}_{h,\epsilon} \leqslant \left(1+\frac{4H\sqrt{dK}}{\epsilon}\right)^{d}\left(1+\frac{8\beta^2\sqrt{d}}{\epsilon^2}\right)^{d^2},\]
          Here the function classes $\overline{\mathcal{Q}}_{h}$ and $ \underline{\mathcal{Q}}_{h}$ are defined in \eqref{eq:def:function:class}. 
        \end{lemma}
        \begin{proof}[Proof of Lemma \ref{lemma:covering}]
          We only estimate the covering number of $\overline{\mathcal{Q}}_{h}$. Suppose $Q_1$ with parameters $(w_1,A_1)$ and $Q_2$ with parameters $(w_2,A_2)$ are in the function class $\overline{\mathcal{Q}}_{h} $, then
          \begin{align*}
            \normof{Q_1-Q_2}_\infty&=\sup\limits_{\phi:\normof{\phi}\leqslant 1}\abso{\mathrm{\Pi}_{H-h+1}\left(\phi^\top w_1 + \beta\sqrt{\phi^\top A_1 \phi}\right)-\mathrm{\Pi}_{H-h+1}\left(\phi^\top w_2 + \beta\sqrt{\phi^\top A_2 \phi}\right)}\\
            &\leqslant\sup\limits_{\phi:\normof{\phi}\leqslant 1}\abso{\left(\phi^\top w_1 + \beta\sqrt{\phi^\top A_1 \phi}\right)-\left(\phi^\top w_2 + \beta\sqrt{\phi^\top A_2 \phi}\right)}\\
            &\leqslant\sup\limits_{\phi:\normof{\phi}\leqslant 1}\abso{\phi^\top(w_1-w_2)} + \sup\limits_{\phi:\normof{\phi}\leqslant 1} \beta\sqrt{\abso{\phi^\top(A_1-A_2)\phi}} \\
            &\leqslant \normof{w_1-w_2}+ \beta\sqrt{\normof{A_1-A_2}_{F}},
          \end{align*} 
          where the second inequality follows from the inequality $\sqrt{x}-\sqrt{y} \leqslant \sqrt{\abso{x-y}}$. Thus $\epsilon/2$-covering of $C_w = \{w \in \mathbb{R}^d:\normof{w}\leqslant H\sqrt{dK}\} $ and $\dfrac{\epsilon^2}{4\beta^2}-$covering of $C_{A} = \{A\in \mathbb{R}^{d\times d}:\normof{A}_{F}\leqslant\sqrt{d}\}$ are sufficient to form an $\epsilon$-cover of $\overline{\cQ}_h$. We obtain the covering number of $\overline{\cQ}_{h}$ satisfies
          \[{\mathcal{N}}_{h,\epsilon} \leqslant  \left(1+\frac{4H\sqrt{dK}}{\epsilon}\right)^{d}\left(1+\frac{8\beta^2\sqrt{d}}{\epsilon^2}\right)^{d^2}.\]
          The inequality follows from the standard bound of the covering number of Euclidean Balls (cf. Lemma 2 in \citet{vershynin2010introduction}).
        \end{proof}
        \begin{lemma}[Concentration for Self-normalized Processes \citep{abbasi2011improved}]\label{lemma:concentration}
          Suppose $\{\epsilon_t\}_{t \geqslant 1}$ is a scalar stochastic
          process generating the filtration $\{\mathcal{F}_t\}_{t \geqslant 1}$, and $\epsilon_t|\mathcal{F}_{t-1}$ is zero mean and $\sigma$-subGaussian. Let $\{\phi_t\}_{t\geqslant 1}$
          be an $\mathbb{R}^d$-valued stochastic process with $\phi_t \in \mathcal{F}_{t-1}$. Suppose $\Lambda_0 \in \mathbb{R}^{d\times d}$ is positive definite, and $\Lambda_t = \Lambda_0 + \sum_{s = 1}^t\phi_s\phi_s^\top$. Then for each $\delta \in (0,1)$, with probability at least $1-\delta$, we have
          \[ \Big\Vert\sum_{s=1}^t \phi_s\epsilon_s \Big\Vert_{\Lambda_t^{-1}}^2\leqslant 2\sigma^2 \log\bigg(\frac{\mathrm{det}(\Lambda_t)^\frac{1}{2}}{\delta\mathrm{det}(\Lambda_0)^\frac{1}{2}}\bigg), \quad \forall t \geqslant 0 . \]
        \end{lemma}

        \begin{proof}
            See \citet{abbasi2011improved} for a detailed proof.
        \end{proof}

        \begin{lemma}[Matix Bernstein Inequality]\label{lemma:matrix_concertra}
           Supposed that $\{A_k\}_{k=1}^{n} $ are independent random matrix in $\mathbb{R}^{d_1\times d_2} $. They satisfy 
          $\expect[A_k] = 0$ and $\normop{A_k}\leqslant L .$
         Let $ Z = \sum_{k=1}^n A_k$ and
         \[v(Z) = \max\big\{\normop{\expect[ZZ^\top]},\normop{\expect[Z^\top Z]}\big\} = \max\bigg\{\normop{\sum_{k=1}^n\expect[A_kA_k^\top]},\normop{\sum_{k=1}^n\expect[A_k^\top A_k]}\bigg\}.\] 
         We have,
         \[\mbp(\normop{Z}\geqslant t)\leqslant (d_1+d_2)\cdot\exp\left(-\dfrac{t^2/2}{v(Z)+L/3\cdot t}\right), \quad \forall t>0. \]
        \end{lemma}
        \begin{proof}
          See Theorem 1.6.2 of \citet{tropp2015introduction} for detailed proof.
        \end{proof}
     
\end{document}